%% file: main.tex
\documentclass[11pt]{article}

\usepackage{style}

\input{commands}
\usepackage{microtype}
\usepackage{pgfplots}
\usepackage{placeins}
\pgfplotsset{compat=1.18}
\usetikzlibrary{arrows.meta}

\makeatletter
\renewcommand*{\theHALG@line}{\thealgorithm.\arabic{ALG@line}}
\makeatother

\title{The Price of Hidden Curvature:\\ Improved Lower Bounds for Bandit Convex Optimization}

\author{
  Nived Rajaraman\\
  \textit{Microsoft Research}\\
  \texttt{nrajaraman@microsoft.com} 
  \and
  Yanjun Han \\
  \textit{New York University} \\
  \texttt{yanjunhan@nyu.edu}
}
\date{}

\begin{document}

\maketitle

\begin{abstract}
\noindent We establish improved lower bounds on the minimax expected regret of stochastic bandit convex optimization for $1$-Lipschitz functions on the $d$-dimensional Euclidean ball. For time horizons $n\ge d^{10/3}$, we prove a lower bound of $\Omega(d^{4/3}\sqrt{n})$, the first nontrivial bound that exceeds the $d\sqrt{n}$ dependence of linear bandits. This shows that stochastic bandit convex optimization is fundamentally harder than linear bandits. For $d^2\le n\le d^{10/3}$, we obtain a lower bound of $\Omega(\sqrt{d}n^{3/4})$, matching the regret of the algorithm of \citet{flaxman2005} and thereby establishing its optimality in this regime.

\medskip
\noindent The hard class of convex functions we construct takes the following form in dimension $2d$: for an action $a = (\aone,\atwo) \in \bbB^{2d}$, each function is the scaled soft maximum of a ``tube'', $r^{-1} \| W^\star \aone - \frac{r}{8\eps} \atwo \|$ (hyperparameterized by $\eps,r$), and a squared distance function, $\frac12 \| \aone - u^\star \|^2 - \frac12 \| u^\star \|^2$. Here $u^\star\in\mathbb R^d$ is the unknown target determining the minimizer, while $W^\star\in\mathbb R^{d\times d}$ hides the region in which the quadratic curvature is observable. Indeed, observations reveal substantial information about $u^\star$ only when the learner acts near the hidden tube $
a^2\approx \frac{8\varepsilon}{r}W^\star a^1$;
away from it, the tube branch dominates and masks the quadratic branch. Thus the learner must pay to uncover the geometry encoded by $W^\star$ before it can effectively exploit the curvature that identifies $u^\star$. Formalizing this tradeoff yields a sample complexity lower bound of
$
\Omega(\frac{d^{5/2}}{\varepsilon^2}\wedge \frac{d^2}{\varepsilon^4})
$
for finding an $\varepsilon$-optimal action, and ultimately the $\Omega(d^{4/3}\sqrt n \wedge \sqrt{d}n^{3/4})$ regret lower bound.

\medskip
\noindent The proof was developed by GPT-5.5 Pro and GPT-5.6 Sol Pro under the authors' guidance.
\end{abstract}


\tableofcontents

\section{Introduction}
\label{sec:introduction}

Stochastic bandit convex optimization asks a learner to minimize an unknown convex function $f$ using only noisy function evaluations. Let
the action space be a convex set denoted $\calA \subseteq \bbR^d$. At round $t$, the learner
chooses an action $a_t \in \calA$ as a function of the past history of interaction and receives the noisy observation,
\begin{equation}
    r_t=f(a_t)+Z_t,
    \qquad
    Z_t\stackrel{\mathrm{i.i.d.}}{\sim}N(0,1).
    \label{eq:observation-model}
\end{equation}
Let $R$ collect all internal randomness used by the learner independent of the unknown convex function and the noise sequence $(Z_t)_{t \ge 1}$, and define the filtration $\calH_t\triangleq\sigma\bigl(R,(a_s,r_s)_{s\le t}\bigr)$ with $\calH_0\triangleq\sigma(R)$. A possibly randomized learner chooses an $\calH_{t-1}$-measurable action $a_t$, and every estimator produced after $t$ rounds is $\calH_t$-measurable.
For a fixed loss $f$ and learner, $\bbE_f$ denotes expectation under the
induced joint law of $R$, the actions, and the observations.

\medskip
\noindent For an action space $\calA\subseteq\bbR^d$, let $\calF_d (\calA)$ be the class of continuous convex functions $f:\calA\to\bbR$ that are $1$-Lipschitz and have a global minimizer $a^\star$ in $\bbB^d$, i.e., the $d$-dimensional unit $\ell_2$ ball. The minimax expected (pseudo-)regret of stochastic bandit convex optimization is defined by
\begin{equation}
    \mathfrak R_n^\star(d ; \calA)
    = \inf_{a^n} \sup_{f\in\calF_d (\calA)}
      \bbE_{f}\left[
        \sum_{t=1}^n
        (f(a_t)-f(a^\star))
      \right],
    \label{eq:minimax-pseudo-regret}
\end{equation}
where the infimum is taken over all possible action sequences $a^n \in \calA^n$ adapted to $(\calH_{t-1})_{t=1}^n$. \cref{eq:minimax-pseudo-regret} marginalizes over observation noise and the learner's randomization. In this paper, we will focus on two settings: $\calA = \bbB^d$ (Euclidean ball action space) and $\calA = \bbR^d$ (unconstrained action space).

\medskip
\noindent Bandit convex optimization has a long history, with a sequence of algorithmic advances steadily improving regret upper bounds \citep{flaxman2005,agarwal2013stochastic,bubeck2015one,bubeckeldan2018,bubeckeldanlee2021}. Lower bounds, however, have seen much less progress: for general Lipschitz convex losses, the strongest dimension-dependent bound remained the $\widetilde\Omega(d\sqrt{n})$ lower bound inherited from linear bandits \citep{dani2008,rusmevichientong2010linearly}. 

\medskip
\noindent In this work, we revisit this
question and establish new lower bounds on the minimax adaptive sample complexity of finding an $\eps$-optimal action for stochastic convex functions. Consequently, we establish regret lower bounds \textit{which strictly improve the best previously known guarantee of $d \sqrt{n}$.}

\begin{theorem}[Estimation Lower Bound]
\label{thm:convex-optimization-lower-bound}
There are universal constants $c_0,c_1>0$ such that the following holds for
all sufficiently large $d$. For every
$\eps \in (0,1/2)$ and every adaptive learner over $\calA = \bbB^d$ that makes $n$ noisy function-value queries and outputs an $\calH_n$-measurable
$\widehat{a}\in\bbB^{d}$, with
\begin{equation}
    n<c_1\begin{cases}
        \frac{d^{2.5}}{\varepsilon^2} &\text{if } \varepsilon \lesssim d^{-1/4} \\
        \frac{d^2}{\varepsilon^4} &\text{if } d^{-1/4} \lesssim \varepsilon \lesssim 1
    \end{cases}, 
    \label{eq:hidden-map-query-budget}
\end{equation}
there exists a function
$f \in \calF_d (\bbB^d)$ such that $\bbE_{f}\left[ f(\widehat{a}) - \min_{a^\star\in \bbB^d} f(a^\star) \right] \ge c_0\eps$. The expectation is over the observation noise and any learner randomization.
\end{theorem}

\begin{theorem}[Regret Lower Bound]\label{thm:convex-bandit-regret-lower-bound}
There are universal constants $c,C>0$ such that, for every $d\ge C$ and $n\ge C d^2$, the minimax regret satisfies
\begin{equation}
    \mathfrak{R}_n^\star(d ; \bbB^d)
    \ge c\begin{cases}
        \sqrt{d}n^{3/4} &\text{if } d^2 \lesssim n \lesssim d^{10/3} \\
        d^{4/3}\sqrt{n} & \text{if } n \gtrsim d^{10/3} 
    \end{cases}.
    \label{eq:convex-bandit-regret-lower-bound}
\end{equation}
\end{theorem}

\noindent We also extend the lower bounds in \Cref{thm:convex-optimization-lower-bound,thm:convex-bandit-regret-lower-bound} to the unconstrained action space setting where $\calA = \bbR^d$ in \Cref{thm:unconstrained-convex-bandit-regret-lower-bound}. In the next section we discuss the main implications of this work and prior connections.

\subsection{Related Work}

There is a long line of work on upper bounds for bandit convex optimization
in both the adversarial and stochastic settings
\citep{agarwal2013stochastic,bubeck2015one,bubeckeldan2018,bubeckeldanlee2021}. In the adversarial setting, among the first known sublinear regret guarantees is the $\calO( \sqrt{d} n^{3/4})$ upper bound of \cite{flaxman2005}.
Currently, the best known information-theoretic and polynomial-time regret bounds are, respectively,
$\widetilde O(d^{5/2}\sqrt{n})$ and $\widetilde O(d^{7/2}\sqrt{n})$
\citep{lattimore2020adversarial,fokkema2024}. In the stochastic setting, for domains including the Euclidean ball, the dimension dependence has improved from
$\widetilde O(d^{9/2}\sqrt{n})$ to $\widetilde O(d^{3/2}\sqrt{n})$, up to
lower-order terms \citep{lattimoregyorgy2021,lattimoregyorgy2023,fokkema2024}.

\medskip
\noindent Before \Cref{thm:convex-bandit-regret-lower-bound}, the strongest general dimension-dependent lower bound for stochastic bandit convex optimization was inherited from stochastic linear bandits.
To make this comparison explicit, for the parameter space $\Theta = \bbB^d$ and action space $\calA = \bbB^d$, define
\begin{equation}
    \mathfrak R_{n,\text{lin}}^\star (d ; \bbB^d)
    = \inf_{a^n} \sup_{\theta \in \Theta}
      \bbE_{\theta}\left[
        \sum_{t=1}^n
        (\langle \theta, a_t\rangle - \min_{a^\star \in \calA} \langle \theta, a^\star\rangle)
      \right],
    \label{eq:minimax-linear-pseudo-regret}
\end{equation}
The minimax regret of linear bandits is known to scale as
\smash{$\mathfrak R_{n,\text{lin}}^\star (d ; \bbB^d) = \widetilde\Theta(n\wedge d\sqrt n)$}
\citep{dani2008,rusmevichientong2010linearly,rajaraman2024,pmlr-v291-zhang25b}. Since linear losses are convex and $1$-Lipschitz under this parameterization, this implies
$\mathfrak R_n^\star(d ; \bbB^d)\ge
\widetilde\Omega(n\wedge d\sqrt n)$. Lower bounds in related settings (smooth convex functions, and strongly convex functions) were also studied by \citet{shamir2013,jamieson2012query,akhavan2020exploiting,akhavan2024gradient}, but did not result in stronger implications for the general Lipschitz convex class. 

\medskip
\noindent For stochastic \textit{first-order}
convex optimization, hard instances are realized by linear functions
\citep{agarwal2009information}. This stems from the inequality $f (x) -  f(x^\star) \le \langle \nabla f(x), x-x^\star \rangle$ for (differentiable) convex functions, which shows that the true suboptimality of a point $x$ (to the minimum $x^\star$) is always dominated by the suboptimality incurred by the local linearization of $f$ around $x$. The same intuition might suggest that even with noisy zeroth-order (i.e., function evaluation) feedback the same behavior holds, and that the optimal regret scales as
$\widetilde\Theta(n\wedge d\sqrt n)$. \Cref{thm:convex-bandit-regret-lower-bound} shows that this intuition surprisingly
fails, separating convex and linear bandits over the Euclidean ball.

\medskip
\noindent Finally, in the regime $d^2 \lesssim n \lesssim d^{10/3}$, our lower bound is matched by the upper bound of \cite{flaxman2005}, who establish a regret upper bound of $\calO (\sqrt{d} n^{3/4})$ when specialized to the Euclidean ball action space. This shows that the minimax regret for convex bandits undergoes at least one new phase transition, switching from linear regret for $n \lesssim d^2$ to $\Theta (\sqrt{d} n^{3/4})$ for $d^2 \lesssim n \lesssim d^{10/3}$ before the eventual $d^{\Theta(1)} \sqrt{n}$ regime. We illustrate different regimes in \Cref{fig:regret}. 

\subsection{Structure and Interpretation of the Lower Bound}
\label{subsec:technical-novelty}

\begin{figure}
    \centering
    \begin{tikzpicture}[thick, scale=2, line cap=round]
\def\plotdimension{2}

  \pgfmathsetmacro{\Tone}{pow(\plotdimension,2)}
  \pgfmathsetmacro{\Tlower}{pow(\plotdimension,3.3333333333)}
  \pgfmathsetmacro{\Tupper}{pow(\plotdimension,4)}
  \pgfmathsetmacro{\Tmax}{1.45*\Tupper}

  \pgfmathsetmacro{\Rone}{pow(\plotdimension,2)}
  \pgfmathsetmacro{\Rlower}{pow(\plotdimension,3)}
  \pgfmathsetmacro{\Rupper}{pow(\plotdimension,3.5)}

  \pgfmathsetmacro{\xscale}{4/\Tmax}
  \pgfmathsetmacro{\Rmax}{1.12*pow(\plotdimension,1.5)*sqrt(\Tmax)}
  \pgfmathsetmacro{\yscale}{2.40/\Rmax}

  \draw[<->] (0,2.38) -- (0,0) -- (4.25,0);
  \node[above] at (0,2.38) {minimax regret};
  \node[below] at (3.85,0) {time horizon $n$};

  \draw[red, dashed]
    ({\xscale*\Tone},0) -- ({\xscale*\Tone},{\yscale*\Rone})
    -- (0,{\yscale*\Rone});
  \node[below, red] at ({\xscale*\Tone},0) {$d^2$};
  \node[left, red] at (0,{\yscale*\Rone}) {$d^2$};

  \draw[red, dashed]
    ({\xscale*\Tlower},0) -- ({\xscale*\Tlower},{\yscale*\Rlower})
    -- (0,{\yscale*\Rlower});
  \node[below, red] at ({\xscale*\Tlower},0) {$d^{10/3}$};
  \node[left, red] at (0,{\yscale*\Rlower}) {$d^3$};

  \draw[blue, dashed]
    ({\xscale*\Tupper},0) -- ({\xscale*\Tupper},{\yscale*\Rupper});
  \node[below, blue] at ({\xscale*\Tupper},0) {$d^4$};

  \def\redtailslope{0.75} 
  \def\redcurvebulge{0.2} 
  
  \draw[red] (0,0)
    -- ({\xscale*\Tone},{\yscale*\Tone});
  \draw[red, domain=\Tone:\Tlower, smooth, variable=\x, samples=200]
    plot ({\xscale*\x},
          {\yscale*sqrt(\plotdimension)*pow(\x,0.75)});

  \draw[red, domain=\Tlower:\Tmax, smooth, variable=\x, samples=200]
  plot ({\xscale*\x},
        {\yscale*(
          \Rlower
          + \redtailslope*
            (pow(\plotdimension,1.3333333333)*sqrt(\x)-\Rlower + \redcurvebulge*
        sin(180*(\x-\Tlower)/(\Tmax-\Tlower)))
        )});

  \draw[blue, dashed] (0,0)
    -- ({\xscale*\Tone},{\yscale*\Tone});
   \draw[blue, dashed, domain=\Tone:\Tlower,
        smooth, variable=\x, samples=200]
    plot ({\xscale*\x},
          {\yscale*sqrt(\plotdimension)*pow(\x,0.75)});
  \draw[blue, domain=\Tlower:\Tupper,
        smooth, variable=\x, samples=200]
    plot ({\xscale*\x},
          {\yscale*sqrt(\plotdimension)*pow(\x,0.75)});
  \draw[blue, domain=\Tupper:\Tmax,
        smooth, variable=\x, samples=200]
    plot ({\xscale*\x},
        {\yscale*(
            (pow(\plotdimension,1.5)*sqrt(\x)+ 0.1*
        sin(180*(\x-\Tupper)/(\Tmax-\Tupper)))
        )});

  \pgfmathsetmacro{\Tshared}{0.60*\Tone+0.40*\Tlower}
  \pgfmathsetmacro{\Rshared}{sqrt(\plotdimension)*pow(\Tshared,0.75)}
  \node[below, xshift=-7pt]
    at ({\xscale*\Tshared*0.87},{\yscale*\Rshared*1.3})
    {$\sqrt{d}\,n^{3/4}$};

  \pgfmathsetmacro{\Tlabel}{0.86*\Tmax}
  \pgfmathsetmacro{\Rupperlabel}{pow(\plotdimension,1.5)*sqrt(\Tlabel)}
  \pgfmathsetmacro{\Rlowerlabel}{pow(\plotdimension,1.3333333333)*sqrt(\Tlabel)}
  \node[above=0.1cm, blue] at ({\xscale*\Tlabel},{\yscale*\Rupperlabel})
    {$d^{3/2}\sqrt{n}$};
    \node[right=1cm, blue] at ({\xscale*\Tlabel},{2.2cm})
    {UB};
  \node[below=0.25cm, red] at ({\xscale*\Tlabel},{\yscale*\Rlowerlabel})
    {$d^{4/3}\sqrt{n}$};
    \node[right=1cm, red] at ({\xscale*\Tlabel},{\yscale*\Rlowerlabel})
    {LB};

  \node[above left, xshift=-1pt, yshift=2pt]
    at ({\xscale*0.64*\Tone},{\yscale*0.44*\Tone}) {$n$};
\end{tikzpicture}
    \caption{Upper and lower bounds on the minimax regret, suppressing logarithmic factors. The red curve indicates the regret lower bound we establish in this work, improving over the prior best known lower bound of $\Omega (d \sqrt{n})$ for $n \gtrsim d^2$. The blue curve indicates the current best known upper bound on the minimax regret, scaling as $\widetilde{\calO} ( n \wedge \sqrt{d} n^{3/4} \wedge d^{3/2} \sqrt{n} )$.}
    \label{fig:regret}
\end{figure}
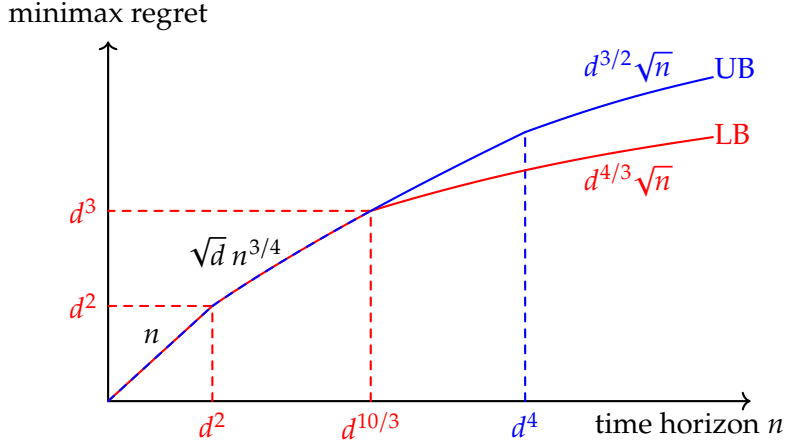

\begin{figure}
\centering
\includegraphics[width=0.66\textwidth,trim={0cm 2cm 0cm 3cm},clip]{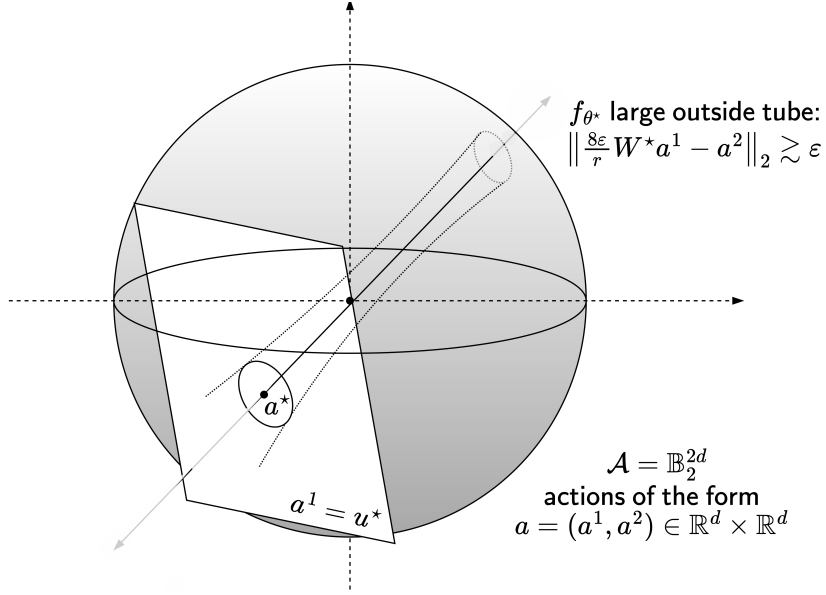}
\caption{\textbf{Geometry of the hard functions $f_{\theta^\star}$.}  The shaded ball is the
action space $\calA=\bbB^{2d}$, with each action split as $a=(\aone,\atwo)$.  The dotted cylinder depicts the narrow tube surrounding the hidden relation $\atwo=(8\eps/r)W^\star\aone$, while the white slice $\aone=u^\star$ is the set on which the second branch of $f_{\theta^\star}$ is minimized. The intersection of this slice with the center of the tube contains the minimizer $a_{\theta^\star}=(u^\star,(8\eps/r)W^\star u^\star)$.  Thus a low-loss action must both locate $u^\star$ and satisfy the tube constraint determined
by $W^\star$.}
\label{fig:hidden-map-geometry}
\end{figure}

In order to interpret \Cref{thm:convex-optimization-lower-bound}, we first describe the structure of the hard functions we study. For simplicity of presentation, we construct the hard instance in dimension $2d$: 
each action is split into two equal parts, $a=(\aone,\atwo) \in \bbB^{2d}$. The ground-truth function is indexed by $\theta^\star=(W^\star,u^\star)$: $u^\star\in\bbR^d$ is an unknown target vector that the learner must find to achieve low regret; the matrix $W^\star\in\bbR^{d\times d}$ is a nuisance parameter which hides information about $u^\star$. The ground-truth convex function $f_{\theta^\star}$ is (throughout $\|\cdot\|$ denotes the $\ell_2$ norm)
\begin{equation}
    f_{\theta^\star}(a) = \eps \cdot \smax \left( \frac1r \big\|  W^\star \aone - \frac{r}{8\eps} \atwo \big\|, \frac12 \big\| \aone-u^\star \big\|^2-\frac12 \| u^\star \|^2 \right).
    \label{eq:intro-hidden-map-softmax}
\end{equation}
Here $\smax(s,t)=\log(e^s+e^t)-\log (2)$ is the softmax function of two inputs, and $u^\star$ is assumed to be of the scale $\| u^\star \| \asymp 1$ and $W^\star$ of the scale $\| W^\star \|_{\mathrm{op}} \asymp 1$. The hyperparameters $r$ and $\eps$ are chosen such that $\eps\in (0,\frac{1}{2})$ and $r\ge 16\eps$. A pictorial depiction of $f_{\theta^\star}$ is in \Cref{fig:hidden-map-geometry}; we summarize several properties of $f_{\theta^\star}$, including convexity and Lipschitzness, in \Cref{sec:construction-properties}. In order to make $f_{\theta^\star}$ small, the learner must choose $\aone \approx u^\star$ and $\atwo$ such that the first branch of the softmax is small, which requires $\atwo \approx \frac{8\eps}{r} W^\star u^\star$.

\medskip
\noindent The structure of $f_{\theta^\star}$ is such that unless $\frac{8 \eps}{r} W^\star \aone \approx \atwo$, the first branch of the softmax dominates, and the function value is large. This means that for any chosen $\aone$, unless the learner's action satisfies
\begin{equation} \label{eq:tube}
    \left\| \frac{8\varepsilon}{r} W^\star\aone - \atwo \right\| \lesssim \varepsilon,
\end{equation}
the first branch within the softmax of $f_{\theta^\star}$ dominates (i.e., is at least a large constant), which occludes information about $u^\star$, \textit{even if the action $\aone$ happened to be informative about $u^\star$.}

\medskip
\noindent We refer to queries that fall within the set \cref{eq:tube}, informally, as the \emph{tube}. The learner can obtain information about $u^\star$ essentially only by playing actions within the tube. If $W^\star$ were known, the learner
can freely enter the center of the tube by taking $\atwo=\frac{8\varepsilon}{r}W^\star\aone$. However, because $W^\star$ is hidden, the learner must instead guess $W^\star\aone$ or learn enough about $W^\star$ to be able to predict this direction for each $\aone$ of its choosing. In the sequel, we focus on the cost of learning an estimator $\widehat{u}$ such that $\| u^\star - \widehat{u} \| \le c$ for some small constant $c>0$. This is a necessary condition to find an action $a$ such that $f_{\theta^\star}(a) - f_{\theta^\star}(a_{\theta^\star}) \le c' \varepsilon$ since $f_{\theta^\star}(a) - f_{\theta^\star}(a_{\theta^\star}) \gtrsim \eps \| \aone - u^\star \|^2$. 

\medskip
\noindent \textit{A simple learning rule.} For a small constant $c_0>0$, pick $d$ orthogonal vectors in $c_0\bbS^{d-1}$, $\calV = (v_i)_{i=1}^d$, uniformly at random, and let $b_i = W^\star v_i$. The learner tries to learn the projections $\beta_i = \langle u^\star, v_i \rangle$ one at a time across $i \in [d]$. In order to do so, the learner first learns an estimator $\widehat{b}_i$ of $b_i$ to some accuracy by querying $f_{\theta^\star}$, to estimate the tube in the direction of $\aone \gets v_i$. The learner can then play actions of the form $(\aone,\atwo)=(v_i,(8\eps/r)\widehat{b}_i)$, which approximately lie in the tube; the resulting observations reveal information about $\beta_i = \langle u^\star, v_i \rangle$. In each iteration, the learner recovers $\beta_i$ to error $\le c/\sqrt{d}$ for some small $c>0$ (let the estimator be denoted $\widehat{\beta}_i$). Finally, combining all the estimates into the vector $\widehat{u} = \sum_{i=1}^d \widehat{\beta}_i v_i$ satisfies $\| \widehat{u} - u^\star \| \le c$ by the Pythagorean theorem.

\medskip
\noindent \textit{Per-iteration sample complexity.} First we calculate the cost of learning $\widehat{b}_i$ for each $i$. We will drop subscripts in this paragraph and refer to $v_i,b_i$ and $\widehat{b}_i$ as $v,b$ and $\widehat{b}$ respectively. Let $E$ denote the learner's target estimation scale for predicting $b$, so that $\norm{b-\widehat b}$ is typically of order $E$ under the posterior. Note that the effective noise variance of the first branch of the softmax is $\sigma^2_1 \asymp r^2/\eps^2$; ignoring logarithmic factors, the $W^\star$\textit{-exploration} cost of learning $b$ to error $E$ is
\begin{equation*}
    N_{W^\star\text{-explore}} (E;r)
    \gtrsim
    \frac{d^2 \sigma^2_1}{E^2} \cdot \bbI(E<1) \asymp \frac{d^2 r^2}{\eps^2 E^2}\cdot \bbI(E<1).
\end{equation*}
Here the indicator is because the learner does not need to carry out any estimation when $E=1$. 
Given such an estimator $\widehat b$, the learner must choose $\atwo$ so that
$\|(r/(8\eps))\atwo-b\|\lesssim r$ in order to learn
$\langle u^\star,v\rangle$ through the second branch of
$f_{\theta^\star}$.  Conditional on the information available to the learner,
let $p_{\mathrm{tube}}(E)$ denote the posterior probability that this query
falls in the tube, where $E$ represents the typical posterior error
$\norm{b-\widehat b}$.  In dimension $d\ge3$, heuristically our analysis bounds the posterior probability of $\atwo$ hitting the tube, i.e., $\{\|(r/(8\eps))\atwo-b\|\lesssim r\}$ by
\begin{equation}
    p_{\mathrm{tube}}(E)
    \lesssim
    \min\left\{1,\frac{r^2}{E^2}\right\}.
    \label{eq:tube-proof-heuristic}
\end{equation}
This is an \emph{anticoncentration} inequality which upper bounds the posterior uncertainty of $b=W^\star v$. At first, this bound may appear optimistic, since if the learner has only localized $b$ to a ball of radius $E$, then finding an action that falls in the tube would seem to occur with probability around $(r/E)^d$. However, what the quadratic dependence in \cref{eq:tube-proof-heuristic} really suggests is a tradeoff between $p_{\mathrm{tube}}$ and the effective SNR: one can shrink $(\aone, \atwo)$ into $\frac{r}{E}(\aone, \atwo)$ to boost the probability of $\|(r/(8\eps))\atwo-W^\star\aone\|\lesssim r$ to $\Omega(1)$, but this scaling also shrinks the Fisher information for $u^\star$ by a factor of $(\frac{r}{E})^2$. This tradeoff is formally established in \cref{eq:anti-concentration}, where $E^2$ is replaced by the directional quantity ${\aone_t}^\top K_{t-1}^{-1}\aone_t$ which represents the squared uncertainty of $W^\star\aone_t$ under the posterior distribution of $W^\star$.  

\medskip
\noindent The learner needs $d\sigma_2^2$ effective observations which fall in the tube to localize \smash{$\langle u^\star, v \rangle$} to error \smash{$\lesssim \frac{1}{\sqrt{d}}$}, where $\sigma_2^2\asymp\varepsilon^{-2}$ is the effective noise variance in the second branch of $f_{\theta^\star}$. The total cost of $u^\star$\textit{-exploration} therefore scales as
\begin{equation*}
    N_{u^\star\text{-explore}} (E;r)
    \gtrsim
    \frac{d \sigma_2^2}{p_{\text{tube}} (E)}
    \gtrsim
    \frac{dE^2}{\eps^2r^2},
\end{equation*}
where in the last inequality, we assume $E \ge r$. Combining both exploration costs gives 
\begin{equation*}
    N_{W^\star\text{-explore}}(E;r)+N_{u^\star\text{-explore}}(E;r)
    \gtrsim
    \frac{d^2r^2}{\eps^2E^2} \cdot \bbI (E < 1)
    +\frac{dE^2}{\eps^2r^2}.
\end{equation*}
Ignoring logarithmic factors, choosing $r \asymp \eps$ and by an application of AM-GM inequality, we get
\begin{equation*}
    N_{W^\star\text{-explore}}(E;r)+N_{u^\star\text{-explore}}(E;r)
    \gtrsim
    \frac{d^2}{E^2} \bbI (E < 1)
    +\frac{dE^2}{\eps^4} \gtrsim \min \left\{ \frac{d^{3/2}}{\eps^2}, \frac{d}{\eps^4} \right\}. 
\end{equation*}
Thus every choice of $E$ costs at least order \smash{$\frac{d^{3/2}}{\eps^2} \wedge \frac{d}{\eps^4}$} to estimate $ \langle u^\star, v_i \rangle$ within accuracy $d^{-1/2}$. Finally, repeating this process across all $d$ target directions, $v_1,\cdots,v_d$, the statistical cost of finding such a \smash{$\widehat{u} = \sum_{i=1}^d \widehat{\beta}_i v_i$} satisfying \smash{$\| \widehat{u} - u^\star \| \le c$} scales as \smash{$\frac{d^{5/2}}{\eps^2} \wedge \frac{d^2}{\eps^4}$}.

\medskip
\noindent \textit{Regret lower bound.} The regret lower bound of $\Omega ( d^{4/3} \sqrt{n})$ follows from the estimation lower bound by accounting for learning costs carefully, as we discuss below. Note that by the scaling of the branches within $f_{\theta^\star}$, a query used to learn \(W^\star v_i\) incurs regret of order \(\varepsilon E/r\), with $E$ being the target estimation error of \(W^\star v_i\), whereas a query used to learn \(u^\star\) incurs regret of order \(\varepsilon\). 
With the choice $r \asymp \eps$, and adding up all $d$ dimensions, the regret of the learner and the total time horizon scale as
\begin{align*}
    \mathfrak{R}_n &= d (E N_{W^\star\text{-explore}} (E;r) + \eps N_{u^\star\text{-explore}} (E;r)) \asymp \frac{d^3}{E} \bbI ( E < 1) + \frac{d^2 E^2}{\eps^3},\\
    n &= d( N_{W^\star\text{-explore}} (E;r) + N_{u^\star\text{-explore}} (E;r)) \asymp \frac{d^3}{E^2} \bbI ( E < 1)  + \frac{d^2 E^2}{\eps^4}. 
\end{align*}
If the learner chooses $E=1$, then $\mathfrak{R}_n \asymp \sqrt{d} n^{3/4}$. If $E < 1$, we have 
\begin{equation*}
    \mathfrak{R}_n  \asymp \frac{d^3}{E} + \frac{d^2 E^2}{\eps^3} \gtrsim \frac{d^3}{E} + \left( \frac{d^3}{E} \right)^{1/3} \left( \frac{d^2 E^2}{\eps^3} \right)^{2/3} \gtrsim d^{4/3} \left(\frac{d^{3/2}}{E} + \frac{d E}{\eps^2}\right) \asymp d^{4/3} \sqrt{n}. 
\end{equation*}
The equality case is given by the choice $E \asymp \eps d^{1/3}$. Altogether, this suggests the regret lower bound $\mathfrak{R}_n \gtrsim d^{4/3} \sqrt{n} \wedge \sqrt{d} n^{3/4}$.


\medskip
\noindent A formal version of the exploration algorithm sketched in this section can be shown to achieve a regret upper bound of $\calOtilde ( \sqrt{d} n^{3/4} \wedge d^{4/3} \sqrt{n} \wedge n)$ when the ground truth is any function belonging to the support of the prior over convex functions considered (cf. \Cref{thm:constructed-family-regret-upper-bound}). This of course does not imply that the minimax regret itself is $\widetilde{\Theta}( \sqrt{d} n^{3/4} \wedge d^{4/3} \sqrt{n} \wedge n)$, but implies that stronger lower bounds can only come by changing the prior over convex functions considered. We discuss these aspects in further detail in \Cref{sec:discussion}.

\paragraph{Proof Organization.}

\Cref{sec:hidden-map-family} introduces the hard Gaussian family and develops the information bottlenecks underlying both lower bounds. We first establish the estimation and regret lower bounds under the Gaussian prior, and then condition this prior on a bounded event to obtain almost surely Lipschitz hard functions and complete the proofs of \Cref{thm:convex-optimization-lower-bound,thm:convex-bandit-regret-lower-bound}. The proofs of the main technical lemmas are given in \Cref{sec:gaussian-sample-proof}. In \Cref{sec:discussion}, we study the tightness of the construction and present a matching regret upper bound for the hard instance constructed in this paper. We then extend both lower bounds to unconstrained action spaces, and discuss limitations of the construction and related open questions. Technical proofs are deferred to \Cref{sec:construction-properties,sec:constructed-family-regret-upper-bound-proof,sec:conic-extension-proofs}.

\paragraph{Statement on AI use.} Both the construction of the hard instance and proofs of lower bounds were developed by GPT-5.5 Pro and GPT-5.6 Sol Pro under the authors' guidance. The authors take full responsibility for the final content.


\section{Proof of Main Results}
\label{sec:hidden-map-family}

\paragraph{Notation.}
For symmetric matrices, $A\preceq B$ denotes the Loewner order.  Throughout, $c,C>0$ denote universal constants that may change from line to line, and all logarithms are base $e$.  Unless explicitly conditioned or subscripted, all expectations and probabilities are under the joint law of the Gaussian prior, the observation noise, and the learner's randomization.  As in \Cref{subsec:technical-novelty}, the ambient dimension is $2d$: actions lie in $\bbR^d\times\bbR^d$, $W^\star\in\bbR^{d\times d}$, and $u^\star\in\bbR^d$. The main lemmas introduced in this section are proved in \Cref{sec:gaussian-sample-proof}.

\medskip
\noindent For the action $a_t=(\aone_t,\atwo_t)$ at round $t$, define the two scalar arguments of the softmax in \cref{eq:intro-hidden-map-softmax} by
\begin{equation*}
    f_t^1
    \triangleq\frac1r \left\| W^\star\aone_t-\frac{r}{8\eps}\atwo_t \right\|,
    \qquad
    f_t^2
    \triangleq\frac12\norm{\aone_t-u^\star}^2
      -\frac12\norm{u^\star}^2.
\end{equation*}
Then the mean loss at time $t$ is $f_{\theta^\star}(a_t) = \eps \, \mathrm{softmax}(f_t^1, f_t^2) = \eps[\log(e^{f_t^1}+e^{f_t^2})-\log 2]$, and the corresponding softmax weight on the second branch of the softmax is the scalar
\begin{equation}
    p_t\triangleq\frac{e^{f_t^2}}{e^{f_t^1}+e^{f_t^2}}\in(0,1).
    \label{eq:softmax-second-branch-weight}
\end{equation}
When $f_t^1$ dominates $f_t^2$, the weight $p_t$ is small, so the query reveals
little information about $u^\star$. Some basic properties of $f_{\theta^\star}$ are summarized in \Cref{sec:construction-properties}. 

\medskip
\noindent The proof begins with the independent Gaussian prior
\begin{equation}
    \operatorname{vec}(W^\star)
    \sim N\left(0,\frac1d I_{d^2}\right),
    \qquad
    u^\star\sim N\left(0,\frac1{16d}I_d\right).
    \label{eq:gaussian-priors}
\end{equation}
For the parameters $(\eps,r)$, we will always choose $\eps<\frac12$ and $r\ge 16\eps$ to ensure the Lipschitz property with high probability; see \Cref{lem:high-probability-lipschitz}. We will first prove the lower bounds under the Gaussian prior, and then truncate the Gaussian prior to ensure Lipschitzness almost surely in \Cref{subsec:bounded-prior-reduction}. 

\medskip
\noindent The proof uses two information matrices. The first matrix is $\overline{\calI}_n^u$ in \cref{eq:target-likelihood-fisher}, the likelihood Fisher information of $u^\star$ averaged over the posterior distribution of $(W^\star, u^\star)$ given $\calH_n$. We will show that a learner which estimates $u^\star$ accurately must have accrued a large amount of $\overline{\calI}_n^u$ along many directions. The second matrix is $K_t$ in \cref{eq:Kt}, the Fisher information matrix of the posterior distribution of $W^\star$ given $\calH_t$. We will use this matrix to characterize the uncertainty in estimating $W^\star$. Finally, we establish an information bottleneck for learning $u^\star$ through learning $W^\star$. Roughly speaking, we will prove an upper bound of $\bbE[\Tr(K_n)]$ showing the difficulty of learning $W^\star$, and also an upper bound of $\bbE[\Tr(\overline{\calI}_n^u K_n^{-1})]$ showing that learning $u^\star$ requires learning $W^\star$ first. The precise definitions and details are explained below. 

\paragraph{Likelihood Fisher Information about $u^\star$.}

Define the positive-semidefinite matrices
\begin{equation}
    \calI_n^u
    \triangleq\sum_{t=1}^n
      \nabla_{u^\star}f_{\theta^\star}(a_t)
      \nabla_{u^\star}f_{\theta^\star}(a_t)^\top
    =\eps^2\sum_{t=1}^n p_t^2\aone_t{\aone_t}^\top,
    \qquad
    \overline{\calI}_n^u
    \triangleq\bbE[\calI_n^u\mid\calH_n].
    \label{eq:target-likelihood-fisher}
\end{equation}
Both matrices in \cref{eq:target-likelihood-fisher} lie in
$\bbR^{d\times d}$.
Here $\calI_n^u$ is the realized (i.e. without taking expectation) likelihood Fisher information about $u^\star$,
while $\overline{\calI}_n^u$ averages it over the posterior distribution of $(W^\star,u^\star)$ given $\calH_n$ and is the
\emph{posterior-averaged likelihood Fisher information}. The matrix $\overline{\calI}_n^u$ tracks the effective amount of information the learner has acquired
about $u^\star$ in different directions. In particular, for a unit vector $v$,
\begin{equation*}
  v^\top\overline{\calI}_n^u v = \eps^2 \cdot \bbE\left[ \sum_{t=1}^n p_t^2\langle \aone_t,v\rangle^2 \ \middle|\ \calH_n \right].
\end{equation*}
Up to the $\eps^2$ factor, the LHS counts the effective queries about $u^\star$ in direction $v$: a query counts proportional to $\langle \aone_t, v \rangle^2$, and the second branch has appreciable ``weight'' (i.e., large $p_t$). Our first lemma argues that if $u^\star$ is localized to constant $\ell_2$ error, then the matrix $\overline{\calI}_n^u$ has many large eigenvalues with a constant probability.

\begin{lemma}[Posterior-Averaged Likelihood Fisher Information Required to Estimate $u^\star$] \label{lem:posterior-target-spread}
Let $\calA = \bbB^{2d}$. There is a universal constant $c_1>0$ such that the following holds.  If an $\calH_n$-measurable estimator $\widehat{u}$ satisfies $\bbE [\norm{\widehat{u}-u^\star}^2] \le c_1$, then there is an event $\calS_n\in\sigma(\calH_n)$ with
$\Pr(\calS_n)\ge1/2$ on which $\overline{\calI}_n^u$ has at least $d/2$
eigenvalues of size at least $d$.
\end{lemma}

\noindent  The key insight behind this lemma is that an estimator which localizes $u^\star$ to a constant-radius $\ell_2$ ball must acquire substantial
information about $u^\star$ along a constant fraction of directions. This is witnessed by the number of eigenvalues of $\overline{\calI}_n^u$ that are at
least of order $d$. More formally, the prior variance of each coordinate of $u^\star$ is of order $1/d$, so nontrivial estimation requires posterior-averaged likelihood Fisher information of order $d$ along a constant fraction of the directions. Since an effective query (with large $p_t$) contributes order $\eps^2$, reaching this threshold costs order $d/\eps^2$ effective queries per direction.


\paragraph{Posterior Fisher Information about $W^\star$.}
For any integer $m\ge1$ and smooth density $\nu$ on $\bbR^m$, let $J(\nu) \triangleq\bbE_{X\sim\nu}\left[ \nabla\log\nu(X)\nabla\log\nu(X)^\top \right] \in\bbR^{m\times m}$ denote the Fisher matrix of the density $\nu$. Let $\rho_{t,W}$ denote the marginal distribution of $W^\star$ given $\calH_t$, integrating out uncertainty in $u^\star$.

\medskip
\noindent We vectorize matrices row by row and identify gradients with respect to
$W\in\bbR^{d\times d}$ with gradients with respect to
$\operatorname{vec}(W)\in\bbR^{d^2}$.  For any
$M\in\bbR^{d^2\times d^2}$, write $M_{ii'}\in\bbR^{d\times d}$ for the block
indexed by output coordinates $i,i'\in\{1,\ldots,d\}$.  Define
\begin{equation*}
    \Trout(M)
    \triangleq\sum_{i=1}^dM_{ii}\in\bbR^{d\times d},
    \qquad
    [\Trout(M)]_{jj'}
    =\sum_{i=1}^dM_{(i,j),(i,j')}.
\end{equation*}
The posterior Fisher matrix of $W^\star$ is
$J(\rho_{t,W})\in\bbR^{d^2\times d^2}$.  For any density $\rho$ on
$\bbR^{d \times d}$, define its \emph{total input Fisher information} matrix and the resulting posterior process by
\begin{equation} \label{eq:Kt}
    K(\rho)\triangleq\frac1{d^2}\Trout(J(\rho))\in\bbR^{d\times d},
    \qquad
    K_t \triangleq K(\rho_{t,W}).
\end{equation}
The matrix $K_t\succeq0$ is indexed by input directions of $W^\star$, and the
normalization gives $K_0=I_d$.  For a unit vector $v\in\bbR^d$,
\begin{equation*}
    v^\top K_t v
    =\frac1{d^2}\sum_{i=1}^d
      \bbE_{\rho_{t,W}}\left[
        \left\langle
          \nabla_{W_{i,:}}\log\rho_{t,W}(W),v
        \right\rangle^2
      \right].
\end{equation*}
The $i$th summand captures directional information about the $i$th
coordinate $(W^\star v)_i$; up to scaling $v^\top K_t v$ is the total
information about the vector $W^\star v$, summed across its $d$
coordinates.  Since the partial trace preserves the trace,
$\Tr(K_t)=d^{-2}\Tr(J(\rho_{t,W}))$ is the total posterior Fisher
information about $W^\star$ up to scaling.

\medskip
\noindent It is clear that $K_t \succeq 0$, and $K_0= I_d$ due to the Gaussian prior of $W^\star$ in \cref{eq:gaussian-priors}. The following lemma summarizes additional properties of $K_t$. 
\begin{lemma}[Properties of $K_t$]\label{lem:Kt-property}
For $d\ge 3$, the following properties hold for $K_t$: 
\begin{enumerate}
    \item $(K_t)_{t\ge 0}$ is a submartingale, i.e. $K_t\preceq \bbE[K_{t+1} \mid \calH_t]$. 
    \item the trace of $K_n$ is upper bounded in expectation: 
    \begin{align}\label{eq:info-constraint}
        \bbE[\Tr(K_n)] \le d + \frac{\varepsilon^2 n}{r^2d^2}. 
    \end{align}
    \item the quantity ${\aone_t}^\top K_{t-1}^{-1} a_t^1$ measures the posterior uncertainty of $W^\star a_t^1$ and controls the weight $p_t$ on the second branch: there exists a universal constant $C>0$ such that
    \begin{align}\label{eq:anti-concentration}
    {\aone_t}^\top K_{t-1}^{-1} a_t^1\cdot \bbE[p_t^2 \Ind_{\|u^\star\|\le 1/2} \mid \calH_{t-1}, a_t] \le Cr^2. 
    \end{align}
\end{enumerate}
\end{lemma}
\noindent In \Cref{lem:Kt-property}, the submartingale property is unsurprising as $\bbE[K_t]$ is an information matrix. However, the reader should note that $K_t \preceq K_{t+1}$ does \emph{not} hold pathwise due to the randomness in the new observation which updates the posterior; this turns out to be a technical issue and will be addressed below. \cref{eq:info-constraint} is a usual upper bound on the information accumulation for $W^\star$, and shows that each round can accumulate $O(\frac{\varepsilon^2}{r^2d^2})$ amount of information. \cref{eq:anti-concentration} is the key motivation behind the introduction of $K_t$. Intuitively, it shows that $\delta^2 := {\aone_t}^\top K_{t-1}^{-1} a_t^1$ is the remaining posterior uncertainty of $W^\star a_t^1$ at the beginning of time $t$. If $\delta \ge r$, then the learner cannot reliably choose $a_t^2$ to ensure a small tube component $f_t^1$, and the probability of success (i.e. to arrive at the target branch) is of order $O(r^2/\delta^2)$, formalizing the heuristic in \cref{eq:tube-proof-heuristic}. Technically, \eqref{eq:anti-concentration} establishes an anti-concentration result of $W^\star a_t^1$ in terms of its Fisher information, and its proof uses a sharp Euclidean Sobolev inequality in \cite{talenti1976}. 

\paragraph{Information Bottleneck for Learning $u^\star$ through $W^\star$.} Based on \Cref{lem:posterior-target-spread} and \Cref{lem:Kt-property}, we can readily see why learning $u^\star$ is hard. \emph{Assume} for now that $K_t\preceq K_n$ for all $t\in [n]$, and ignore the high-probability indicator $\Ind_{\|u^\star\|\le 1/2}$ in \cref{eq:anti-concentration}. Summing \cref{eq:anti-concentration} over $t=1,\dots,n$ and using the definition of $\overline{\calI}_n^u $ in  \cref{eq:target-likelihood-fisher}, we have
\begin{align}\label{eq:tension-intuition}
\bbE[\Tr(\overline{\calI}_n^u K_n^{-1})] \lesssim n\varepsilon^2 r^2. 
\end{align}
Intuitively, \cref{eq:tension-intuition} means that learning $u^\star$ (i.e. a large $\overline{\calI}_n^u$) requires learning $W^\star$ (i.e. a large $K_n$). However, \cref{eq:info-constraint} shows that the total information for learning $W^\star$ is limited. As a consequence, the matrix $\overline{\calI}_n^u$ cannot be large along too many directions, which by \Cref{lem:posterior-target-spread} shows that an accurate estimation of $u^\star$ is impossible. 

\medskip
\noindent To formalize the intuition in \cref{eq:tension-intuition}, a technical difficulty is that $K_t \preceq K_n$ does not hold pathwise. The next lemma addresses this issue by introducing an orthogonal projection matrix $Q_n$. 

\begin{lemma}[Information bottleneck for estimation complexity]\label{lem:tension-estimation}
Let $\calA=\bbB^{2d}$ and $d\ge 3$. For every horizon $n$, there exist universal constants $c, C>0$ and an $\calH_n$-measurable orthogonal projection $Q_n\in \bbR^{d\times d}$ such that $\bbE[\Tr(I_d - Q_n)] \le \frac{d}{16}$ and
\begin{align}
    \bbE[\Tr(\overline{\calI}_n^u Q_n)] &\le C n\varepsilon^2(r^2+e^{-cd})\cdot \left( 1+\frac{\varepsilon^2 n}{r^2d^3}\right). \label{eq:tension}
\end{align}
\end{lemma}

\noindent Recall that an orthogonal projection matrix $Q$ satisfies $Q=Q^\top$ and $Q^2 = Q$, and therefore has all eigenvalues $0$ or $1$ and satisfies $\Tr(Q) = \mathrm{rank}(Q)$. To see how \Cref{lem:tension-estimation} formalizes the intuition in \cref{eq:tension-intuition}, note that the RHS of \cref{eq:tension} is essentially the RHS of \cref{eq:tension-intuition} (with an additional exponential term $e^{-cd}$ that accounts for the low probability event $\norm{u^\star}>1/2$) times the upper bound of $\Tr(K_n)/d$ in \cref{eq:info-constraint}.  

\medskip
\noindent We also briefly comment on how $Q_n$ is constructed in the proof of \Cref{lem:tension-estimation}. Initialize $Q_0=I_d$. Suppose $Q_{t-1}$ is constructed and $K_t$ is computed at the end of time $t$, perform the eigendecomposition of $K_t$ restricted to the range of $Q_{t-1}$ (or equivalently, the eigendecomposition of $Q_{t-1}K_tQ_{t-1}$). We remove all eigendirections whose eigenvalue exceeds a given factor $\Lambda$ from $Q_{t-1}$ to form $Q_t$. In this way, we arrive at a sequence $Q_0 \succeq Q_1 \succeq \dots \succeq Q_n$, and argue from \cref{eq:info-constraint} that for an appropriately chosen $\Lambda$, there are still many remaining directions in $Q_n$, i.e. $\bbE[\Tr(I_d - Q_n)] \le \frac{d}{16}$. Moreover, this construction ensures that $Q_tK_tQ_t\preceq \Lambda Q_t$ for every $t$, and as a consequence, it holds that
\begin{align*}
{\aone_t}^\top K_{t-1}^{-1} a_t^1 \ge \frac{1}{\Lambda}\|Q_{t-1}a_t^1\|^2 \ge \frac{1}{\Lambda} {\aone_t}^\top Q_n a_t^1. 
\end{align*}
Therefore, by moving from $K_t$ to $Q_t$, we overcome the issue that $K_t$ may not be pathwise increasing. Using the above inequality, summing \cref{eq:anti-concentration} over $t=1,\dots,n$, and plugging in the choice of $\Lambda$, we will arrive at the target inequality \cref{eq:tension}. 

\medskip
\noindent To establish the regret lower bound, we make a slight modification to \Cref{lem:tension-estimation}. 

\begin{lemma}[Information bottleneck for regret]\label{lem:tension-regret}
Let $\calA=\bbB^{2d}$ and $d\ge 3$, and suppose a learner achieves an expected regret $\mathfrak{R}_n$ averaged over the randomness in $(W^\star,u^\star,r_1,\dots,r_n)$ over a time horizon $n$. There exist universal constants $c, C>0$ and an $\calH_n$-measurable orthogonal projection $Q_n\in \bbR^{d\times d}$ such that $\bbE[\Tr(I_d - Q_n)] \le \frac{d}{16}$ and
\begin{align}
    \bbE[\Tr(\overline{\calI}_n^u Q_n)] \le Cn\varepsilon^2(r^2+e^{-cd})\cdot \left( 1+\frac{\eps^2 \mathfrak{R}_n^2}{r^2 d^6}\right). \label{eq:tension-2}
\end{align}
\end{lemma}

\noindent Compared with \Cref{lem:tension-estimation}, the main difference in \Cref{lem:tension-regret} is that the quantity $n$ in \cref{eq:tension} is replaced by a function of the regret $\mathfrak{R}_n$ in \cref{eq:tension-2}. Intuitively, such a replacement is obtained by the observation that the instantaneous regret can be lower bounded by a function of the tube error $\|W^\star a^1 - \frac{r}{8\varepsilon}a^2\|$. To arrive at the precise form of \cref{eq:tension-2}, in the proof we define a counterpart $\widetilde{K}_t$ of $K_t$ which only measures the information along the unlearned directions given by $Q_{t-1}$, and lower bound the regret $\mathfrak{R}_n$ through these unlearned directions. We leave these details to \Cref{sec:gaussian-sample-proof}.

\subsection{Estimation Lower Bound under the Gaussian Prior}

In this section, we prove a simplified version of \Cref{thm:convex-optimization-lower-bound} under the Gaussian prior in \cref{eq:gaussian-priors}, using the information bottleneck established in \Cref{lem:tension-estimation}. 

\begin{theorem}[Estimation Lower Bound under the Gaussian Prior]
\label{thm:gaussian-hidden-map-estimation}
Let $\calA = \bbB^{2d}$ and $d\ge 3$ be large enough. Under the prior in \cref{eq:gaussian-priors}, with $f_{\theta^\star}$ defined in
\cref{eq:intro-hidden-map-softmax}, there are universal constants $c, c'>0$ such that for every hyperparameter $(\eps,r)$ with $0<\eps<\frac{1}{2}$ and $r\ge \max\{16\eps,d^{-100}\}$, every adaptive learner that produces an
$\calH_n$-measurable action $\widehat{a}\in \calA$ with $\bbE[f_{\theta^\star}(\widehat{a})-f_{\theta^\star}(a_{\theta^\star})]\le c'\varepsilon$ must satisfy
\begin{equation*}
  n\ge c\min\left\{\frac{d^{5/2}}{\eps^2}, \frac{d^2}{\varepsilon^2 r^2}\right\}.
\end{equation*}
In particular, any choice of $r$ with $\max\{16\eps,d^{-100}\} \le r\le \max\{16\eps, d^{-1/4}\}$ gives an estimation lower bound
\begin{equation*}
  n\ge c\min\left\{\frac{d^{5/2}}{\eps^2}, \frac{d^2}{\varepsilon^4}\right\}.
\end{equation*}
\end{theorem}

\begin{proof}
It suffices to show the first statement. We first show that if $n< c\min\{\frac{d^{5/2}}{\eps^2}, \frac{d^2}{\varepsilon^2 r^2}\}$ for a sufficiently small constant $c>0$, then every $\calH_n$-measurable estimator $\widehat{u}$ must satisfy $\bbE\norm{\widehat{u}-u^\star}^2\ge c_1$, with constant $c_1>0$ given by \Cref{lem:posterior-target-spread}. Assume the contrary, then \Cref{lem:posterior-target-spread} gives an event $\calS_n$ with $\bbP(\calS_n)\ge \frac{1}{2}$, and $\overline{\calI}_n^u$ has at least $d/2$ eigenvalues of size at least $d$ on $\calS_n$. Next, \Cref{lem:tension-estimation} gives an $\calH_n$-measurable orthogonal projection matrix $Q_n$ such that $\bbE[\Tr(I_d-Q_n)]\le \frac{d}{16}$. By Markov's inequality, $\mathrm{rank}(Q_n) \ge \frac{3d}{4}$ with probability at least $\frac{3}{4}$; call this event $\calE_n$. The union bound gives $\bbP(\calE_n\cap \calS_n)\ge \frac{1}{4}$, and on $\calE_n\cap \calS_n$, von Neumann's trace inequality yields $\Tr(\overline{\calI}_n^u Q_n) \ge \frac{d^2}{4}$ and
\begin{align*}
\bbE[\Tr(\overline{\calI}_n^u Q_n)] \ge \bbE[\Tr(\overline{\calI}_n^u Q_n) \Ind_{\calE_n\cap \calS_n}] \ge \bbE[ \frac{d^2}{4} \Ind_{\calE_n\cap \calS_n}] \ge \frac{d^2}{16}. 
\end{align*}
This is a contradiction to \cref{eq:tension} and the assumption $n< c\min\{\frac{d^{5/2}}{\eps^2}, \frac{d^2}{\varepsilon^2 r^2}\}$ with a small enough $c>0$; note that the assumption of $r\ge d^{-100}$ ensures $e^{-cd}=O(r^2)$ in \cref{eq:tension}. Therefore, $\bbE\norm{\widehat{u}-u^\star}^2\ge c_1$ for every $\calH_n$-measurable estimator $\widehat{u}$. 

\medskip
\noindent Next, assume that an $\calH_n$-measurable action $\widehat{a}\in \calA$ satisfies $\bbE[f_{\theta^\star}(\widehat{a})-f_{\theta^\star}(a_{\theta^\star})]\le c'\varepsilon$. By \cref{eq:target-excess-mse}, this implies $\bbE[\|\widehat{a}^1 - u^\star\|^2 \Ind_{\|u^\star\|\le 2}]\le Cc'$. In addition, $\bbE[\|\widehat{a}^1 - u^\star\|^2 \Ind_{\|u^\star\|> 2}] \le 2\bbE[(1+\|u^\star\|^2)\Ind_{\|u^\star\|> 2}] \le c_2e^{-c_3d}$, and therefore $\bbE[\|\widehat{a}^1 - u^\star\|^2]\le Cc'+ c_2e^{-c_3d}$. For $c'>0$ small enough and $d$ large enough, the right-hand side is smaller than $c_1$, a desired contradiction.  
\end{proof}

\subsection{Regret Lower Bound under the Gaussian Prior}
Similarly, we prove a simplified version of \Cref{thm:convex-bandit-regret-lower-bound} under the Gaussian prior in \cref{eq:gaussian-priors}, using the information bottleneck established in \Cref{lem:tension-regret}. 

\begin{theorem}[Bayes Regret Lower Bound under the Gaussian Prior]
\label{thm:gaussian-hidden-map-regret}
Let $\calA = \bbB^{2d}$ and $d\ge 3$ be large enough. Under the prior in \cref{eq:gaussian-priors}, with $f_{\theta^\star}$ defined in
\cref{eq:intro-hidden-map-softmax}, there is a universal constant
$c$ such that for every hyperparameter $(\eps,r)$ with $0<\eps<\frac{1}{2}$ and $r\ge \max\{16\eps,d^{-100}\}$, every adaptive learner over a time horizon $n$ must suffer from an expected regret
\begin{equation*}
  \mathfrak{R}_n \ge c \min\left\{ n\eps, \frac{d^{8/3}}{\eps}, \frac{d^2}{\eps r^2}\right\}. 
\end{equation*}
In particular, for 
\begin{align} \label{eq:eps-choice}
\varepsilon = 
 c'\begin{cases}
    \sqrt{d}n^{-1/4} & \text{if } d^2 \le n \le d^{10/3}\\
    d^{4/3}n^{-1/2} & \text{if } n > d^{10/3}
\end{cases}
\end{align}
with a small constant $c'>0$ and any choice of $r$ with $\max\{16\eps,d^{-100}\} \le r\le \max\{16\eps, d^{-1/3}\}$, we have
\begin{equation*}
  \mathfrak{R}_n \ge c \min\left\{ \sqrt{d}n^{3/4}, d^{4/3}\sqrt{n} \right\}. 
\end{equation*}
\end{theorem}

\begin{proof}
Again it suffices to prove the first statement; the second statement follows from algebra. Let
\begin{align*}
m \triangleq \Big\lceil c_0\min\left\{n, \frac{d^{8/3}}{\eps^2}, \frac{d^2}{\eps^2 r^2} \right\} \Big\rceil \le n,
\end{align*}
for a small enough constant $c_0<1$. Since $\mathfrak{R}_n \ge \mathfrak{R}_m$, it suffices to show that $\mathfrak{R}_m\ge cm\eps$. Suppose the contrary that some learner achieves an expected regret $\mathfrak{R}_m < cm\eps$ for a sufficiently small constant $c>0$ under the Gaussian prior. By \cref{eq:target-excess-mse}, we have
\begin{align*}
\mathfrak{R}_m \ge \sum_{t=1}^m c_2 \varepsilon \bbE[\|\aone_t - u^\star\|^2\Ind_{\|u^\star\|\le 2}] \ge c_2\varepsilon m\cdot \bbE\left[\Big\| \frac{1}{m}\sum_{t=1}^m \aone_t - u^\star \Big\|^2 \Ind_{\|u^\star\|\le 2}\right]. 
\end{align*}
This means that the learner's average action $\widehat{u} = \frac{1}{m}\sum_{t=1}^m a_1^t$ satisfies $\bbE\|\widehat{u}-u^\star\|^2 \le \frac{\mathfrak{R}_m}{c_2\varepsilon m} + \bbE[\|\widehat{u}-u^\star\|^2\Ind_{\|u^\star\|>2}] < \frac{c}{c_2} + c_3e^{-c_4d}$, which is at most the constant $c_1$ in \Cref{lem:posterior-target-spread} for $c>0$ small enough and $d$ large enough. Therefore, \Cref{lem:posterior-target-spread} gives an $\calH_m$-measurable event $\calS_m$ with $\bbP(\calS_m) \ge \frac{1}{2}$, such that on $\calS_m$, $\overline{\calI}_m^u$ has at least $d/2$ eigenvalues of size at least $d$. 

\medskip
\noindent Next, by \Cref{lem:tension-regret}, there exists an $\calH_m$-measurable orthogonal projection matrix $Q_m$ such that $\bbE[\Tr(I_d-Q_m)]\le \frac{d}{16}$. By Markov's inequality, $\mathrm{rank}(Q_m) \ge \frac{3d}{4}$ with probability at least $\frac{3}{4}$; call this event $\calE_m$. Now using the same analysis in \Cref{thm:gaussian-hidden-map-estimation}, we have 
\begin{align*}
\bbE[\Tr(\overline{\calI}_m^u Q_m)] \ge \bbE[\Tr(\overline{\calI}_m^u Q_m) \Ind_{\calE_m\cap \calS_m}] \ge \bbE \left[ \frac{d^2}{4} \Ind_{\calE_m\cap \calS_m} \right] \ge \frac{d^2}{16}. 
\end{align*}
On the other hand, by \cref{eq:tension-2} and the assumption $\mathfrak{R}_m \le cm\eps$, we have (note that $e^{-cd}=O(r^2)$)
\begin{align*}
\bbE[\Tr(\overline{\calI}_m^u Q_m)] \le Cm\eps^2r^2\left( 1+\frac{\eps^2 \mathfrak{R}_m^2}{r^2 d^6}\right) \le C\eps^2 r^2m + Cc^2\frac{\eps^6 m^3}{d^6}. 
\end{align*}
Using the definition of $m$ and choosing $c,c_0>0$ small enough, the above two displays give the desired contradiction.   
\end{proof}

\subsection{Completing the Proofs: Moving from Gaussian Prior to Bounded Prior}\label{subsec:bounded-prior-reduction}
\Cref{thm:gaussian-hidden-map-estimation,thm:gaussian-hidden-map-regret} are established under the Gaussian priors on $(W^\star,u^\star)$, where the event $\calE_{\mathrm{b}}$ in \cref{eq:bounded-prior-event} that ensures $1$-Lipschitzness of $f_{\theta^\star}$ is violated with exponentially small but non-zero probability. To complete the proofs of \Cref{thm:convex-optimization-lower-bound,thm:convex-bandit-regret-lower-bound}, we establish lower bounds for a bounded prior $\pi_{\mathrm{b}}$, which is the restriction of the Gaussian prior
\cref{eq:gaussian-priors} to the event $\calE_{\mathrm{b}}$. By \Cref{lem:high-probability-lipschitz}, the function $f_{\theta^\star}$ is $1$-Lipschitz almost surely under $\pi_{\mathrm{b}}$. 

\medskip
\noindent The key to the reduction is the following technical lemma. 

\begin{lemma}[Transferring Guarantees from $\pi$ to $\pi_{\mathrm{b}}$]
\label{lem:conditioning-transfer}
Let $\pi$ denote the Gaussian prior in \cref{eq:gaussian-priors}.  Fix a
learning algorithm $\Alg$ with $\calH_{t-1}$-measurable actions for $t\le n$
and an arbitrary $\calH_n$-measurable estimate of $u^\star$. Run $\Alg$ in
the following two settings:
\begin{enumerate}[(i)]
    \item under $\bbP_{\mathrm{G}}$, sample $(W^\star,u^\star)\sim\pi$;
    \item under $\bbP_{\mathrm{b}}$, sample
    $(W^\star,u^\star)\sim\pi_{\mathrm{b}}=\pi(\,\cdot\mid\calE_{\mathrm{b}})$.
\end{enumerate}
In both experiments, use the observation model
\cref{eq:observation-model}, the same decision and output rules, and the same
law for the learner's internal randomness. Let $\widehat{u}^q$ denote the
resulting output under $\bbP_q$. Let $\Pi_2$ denote Euclidean projection onto
$2\bbB^d$, and write, for $q\in\{\mathrm{G},\mathrm{b}\}$, $\overline{u}^{q}\triangleq\Pi_2(\widehat{u}^{q})$. Then
\begin{equation}
    \bbE_{\mathrm{G}} \big[ \norm{\overline{u}^{\mathrm{G}}-u^\star}^2 \big]
    \le
      \bbE_{\mathrm{b}} \big[ \norm{\widehat{u}^{\mathrm{b}}-u^\star}^2 \big]
      +Ce^{-cd}.
    \label{eq:conditioning-transfer}
\end{equation}
Similarly, let $\pi^{\mathrm{G}}$ and $\pi^{\mathrm{b}}$ be the resulting policy under the Gaussian and bounded priors, respectively, then their regrets $\mathfrak{R}_n$ satisfy
\begin{align}\label{eq:regret-transfer}
\bbE_{\mathrm{G}}[\mathfrak{R}_n(\pi^{\mathrm{G}})]
    \le \bbE_{\mathrm{b}}[\mathfrak{R}_n(\pi^{\mathrm{b}})] +C\frac{n\varepsilon}{r\wedge 1}e^{-cd}.
\end{align}
\end{lemma}

\begin{proof}
Conditional on $\calE_{\mathrm{b}}$, the parameter has law $\pi_{\mathrm{b}}$. Given the parameter, both experiments use the same learning algorithm, output rule, random-seed law, and observation kernel. Hence the entire adaptive trajectory satisfies
\begin{equation*}
    \operatorname{Law}_{\mathrm{G}}
      \bigl(W^\star,u^\star,R,(a_t^{\mathrm{G}},r_t^{\mathrm{G}})_{t\le n}
        \mid\calE_{\mathrm{b}}\bigr)
    =
    \operatorname{Law}_{\mathrm{b}}
      \bigl(W^\star,u^\star,R,
        (a_t^{\mathrm{b}},r_t^{\mathrm{b}})_{t\le n}\bigr).
\end{equation*}
Splitting the Gaussian risk over
$\calE_{\mathrm{b}}$ and $\calE_{\mathrm{b}}^c$ gives
\begin{align}
    \bbE_{\mathrm{G}} \big[ \norm{\overline{u}^{\mathrm{G}}-u^\star}^2 \big]
    &=
      \bbP_{\mathrm{G}}(\calE_{\mathrm{b}}) \cdot
      \bbE_{\mathrm{b}} \big[ \norm{\overline{u}^{\mathrm{b}}-u^\star}^2 \big] 
      + \bbE_{\mathrm{G}} \big[
        \norm{\overline{u}^{\mathrm{G}}-u^\star}^2
        \, \Ind_{\calE_{\mathrm{b}}^c}
      \big] \nonumber\\
    &\le
      \bbE_{\mathrm{b}} \big[ \norm{\widehat{u}^{\mathrm{b}}-u^\star}^2 \big]
      + \bbE_{\mathrm{G}} \big[
        \norm{\overline{u}^{\mathrm{G}}-u^\star}^2 \,
        \Ind_{\calE_{\mathrm{b}}^c}
      \big]. \label{eq:0004}
\end{align}
For the inequality step, note that under $\bbP_{\mathrm{b}}$,
$\norm{u^\star}\le \frac{3}{8}<2$, so projection onto $2\bbB^d$ cannot increase
the distance to $u^\star$. It remains to bound the second term on the RHS of \cref{eq:0004}. Since
$\norm{\overline{u}^{\mathrm{G}}}\le2$,
\begin{equation*}
\begin{aligned}
    \bbE_{\mathrm{G}}\!\left[
      \norm{\overline{u}^{\mathrm{G}}-u^\star}^2
      \Ind_{\calE_{\mathrm{b}}^c}
    \right] 
    &\le
      \bigl(8+2\bbE_{\mathrm{G}}\norm{u^\star}^2\bigr)
      \bbP_{\mathrm{G}}\!\left(\norm{W^\star}_{\mathrm{op}}>4\right)
      +\bbE_{\mathrm{G}}\!\left[
        \bigl(8+2\norm{u^\star}^2\bigr)
        \Ind_{\{\norm{u^\star}>3/8\}}
      \right] \\
    &\le Ce^{-cd}.
\end{aligned}
\end{equation*}
This proves the inequality in
\cref{eq:conditioning-transfer}. Similarly, for the regret transfer, we have
\begin{align*}
\bbE_{\mathrm{G}}[\mathfrak{R}_n(\pi^{\mathrm{G}})]
    = \bbP_{\mathrm{G}}(\calE_{\mathrm{b}})\cdot \bbE_{\mathrm{b}}[\mathfrak{R}_n(\pi^{\mathrm{b}})] + \bbE_{\mathrm{G}}\big[\mathfrak{R}_n(\pi^{\mathrm{G}})\Ind_{\calE_{\mathrm{b}}^c} \big]. 
\end{align*}
Since for any policy $\pi$ we have
\begin{align*}
\mathfrak{R}_n(\pi) \le n\left(\frac{\varepsilon}{r}(1+\|W^\star\|_{\mathrm{op}}) + \varepsilon(1+\|u^\star\|) \right), 
\end{align*}
the same tail bounds give the remainder term in \cref{eq:regret-transfer}. 
\end{proof}

\paragraph{Proof of \Cref{thm:convex-optimization-lower-bound}.} Suppose that there exists an $\calH_n$-measurable learner $\widehat{a}$ such that $\bbE_{\theta^\star}[f_{\theta^\star}(\widehat{a}) - f_{\theta^\star}(a_{\theta^\star})]\le c_0\varepsilon$ for a small $c_0$ and every $\theta^\star$ supported on $\pi_{\mathrm{b}}$. Since $\|u^\star\|\le \frac{3}{8}$ under $\pi_{\mathrm{b}}$, \cref{eq:target-excess-mse} implies that $\bbE_{\mathrm{b}}[\|\widehat{a}^1 - u^\star\|^2] \le Cc_0$. By \Cref{lem:conditioning-transfer}, this gives an estimator $\widehat{u}$ under the Gaussian prior such that  $\bbE_{\mathrm{G}}[\|\widehat{u} - u^\star\|^2] \le Cc_0 + C'e^{-cd}$. For small enough $c_0$ and large enough $d$, this error is small enough for the proof of \Cref{thm:gaussian-hidden-map-estimation} to hold, and thus \Cref{thm:convex-optimization-lower-bound} is proved. 

\paragraph{Proof of \Cref{thm:convex-bandit-regret-lower-bound}.} By \cref{eq:regret-transfer} in \Cref{lem:conditioning-transfer}, we only need to show that for the choice of $(\varepsilon,r)$ in \Cref{thm:gaussian-hidden-map-regret}, the remainder term $Cn\varepsilon e^{-cd}/(r\wedge 1)$ is dominated by $\min\{\sqrt{d}n^{3/4}, d^{4/3}\sqrt{n}\}$. This can be easily verified: for $\eps$ in \cref{eq:eps-choice} and $r\ge d^{-100}$, we have $n\eps e^{-cd}/(r\wedge 1)\le n\eps \cdot d^{100}e^{-cd}\le \kappa n\eps \le \kappa' \min\{\sqrt{d}n^{3/4}, d^{4/3}\sqrt{n}\}$, where the constants $\kappa, \kappa'>0$ can be made arbitrarily small for $d$ large enough. Therefore \Cref{thm:convex-bandit-regret-lower-bound} is also proved.

\paragraph{Restriction on $d$.} The current construction requires the dimension of the action space to be $2d$, with $d$ large enough. Both constraints can be relaxed easily. First, the case of odd dimensions can be reduced to the even case by making one coordinate of the action vector ineffective. Second, when $d=O(1)$, both the estimation complexity lower bound $\Omega(\varepsilon^{-2})$ and regret lower bound $\Omega(\sqrt{n})$ are standard. 

\section{Proofs of Main Lemmas}
\label{sec:gaussian-sample-proof}

We prove the main lemmas in this section.

\subsection{Proof of \Cref{lem:posterior-target-spread}} \label{sec:posterior-target-spread}

For each realized $\calH_n$, the posterior mean is the best predictor of $u^\star$, and its conditional risk is
\begin{equation}
    \Tr( \Cov(u^\star\mid\calH_n) ) = \inf_{\widetilde{u} \in\bbR^d}
    \bbE \left[
      \norm{\widetilde{u}-u^\star}^2
      \,\middle|\,
      \calH_n
    \right]
    \le \bbE \big[ \norm{\widehat u-u^\star}^2 \, \big| \, \calH_n \big]. \label{eq:oracle-posterior-risk}
\end{equation}
Here, $\widehat{u}$ is any $\calH_n$-measurable estimator. If $\widehat{u}$ achieves a small mean square error $\bbE[\norm{\widehat u-u^\star}^2] \le c_1$, Markov's inequality combined with \cref{eq:oracle-posterior-risk} gives $\Tr( \Cov(u^\star\mid\calH_n) ) \le 4c_1$ with probability at least $\frac{3}{4}$.

\medskip
\noindent Next, the matrix Cram\'er--Rao inequality~\citep{dembo1991information} applied to the conditional
posterior density relates $\Cov(u^\star\mid\calH_n)$ to the
conditional posterior Fisher information about $u^\star$:
\begin{equation*}
\Cov(u^\star\mid\calH_n)
\succeq
J(\rho_{n,u})^{-1},
\end{equation*}
where $\rho_{n,u}$ is the marginal posterior density of $u^\star$ given $\calH_n$, integrating out $W^\star$. We prove the following lemma. 

\begin{lemma}\label{lem:likelihood-Fisher}
There is an absolute constant $C>0$ such that with probability at least $\frac{3}{4}$, 
\begin{align*}
J(\rho_{n,u}) \preceq C\left(dI_d + \overline{\calI}_n^u \right). 
\end{align*}
\end{lemma}
\begin{proof}
We first show that
\begin{equation}
    J(\rho_{n,u})
    \preceq 16dI_d
      +\overline{\calI}_n^u
      -\overline{\mathcal C}_n^u,
    \label{eq:target-posterior-fisher}
\end{equation}
with
\begin{equation} \label{eq:curvature-correction}
  \mathcal C_n^u
  \triangleq
  \eps\sum_{t=1}^n
  Z_t p_t(1-p_t)\aone_t{\aone_t}^\top,
  \qquad
  \overline{\mathcal C}_n^u
  \triangleq
  \bbE[\mathcal C_n^u\mid\calH_n].
\end{equation}
To show \cref{eq:target-posterior-fisher}, 
let $q_{n,u}$ be the conditional posterior density of $u^\star$ given $\calH_n$: 
\begin{equation*}
    q_{n,u}(u)
    \propto
    \bbE_{W^\star} \exp\left(-8d\norm{u}^2 + \ell_n(W^\star, u)\right).
\end{equation*}
Here $\ell_n(W,u)=\log p(\calH_n\mid W,u)$ is the log-likelihood function. Standard differentiation under the integral gives
\begin{align*}
-\nabla_u^2 \log q_{n,u}(u) &= 16dI_d + \bbE[-\nabla_u^2 \ell_n(W^\star,u)\mid \calH_n, u] - \Cov(\nabla_u \ell_n(W^\star,u)\mid \calH_n, u) \\
&\preceq 16dI_d + \bbE[-\nabla_u^2 \ell_n(W^\star,u)\mid \calH_n, u]. 
\end{align*}
To differentiate, for a generic $u\in\bbR^d$, let
\begin{equation*}
    f_t^2(u)
    \triangleq
    \frac12\norm{\aone_t-u}^2-\frac12\norm{u}^2
    =
    \frac12\norm{\aone_t}^2-\langle\aone_t,u\rangle,
    \qquad
    p_t(u)
    \triangleq
    \frac{e^{f_t^2(u)}}{e^{f_t^1}+e^{f_t^2(u)}}.
\end{equation*}
In particular, $p_t(u^\star)=p_t$.  Since
$\nabla_u f_t^2(u)=-\aone_t$ and $\nabla_u^2f_t^2(u)=0$, differentiating the
softmax gives
\begin{equation*}
    \nabla_u f_{(W^\star,u)}(a_t)
    =
    -\eps p_t(u)\aone_t,
    \qquad
    \nabla_u^2 f_{(W^\star,u)}(a_t)
    =
    \eps p_t(u)(1-p_t(u))\aone_t{\aone_t}^\top.
\end{equation*}
For a single observation $(a_t,r_t)$, the negative Hessian of the log-likelihood term is
\begin{equation*}
\begin{aligned}
    -\nabla_u^2\left[
      -\frac12\bigl(r_t-f_{(W^\star,u)}(a_t)\bigr)^2
    \right]
    =
    \nabla_u f_{(W^\star,u)}(a_t)
    \nabla_u f_{(W^\star,u)}(a_t)^\top 
    -\bigl(r_t-f_{(W^\star,u)}(a_t)\bigr)
      \nabla_u^2 f_{(W^\star,u)}(a_t).
\end{aligned}
\end{equation*}
A combination of the above gives $\bbE[-\nabla_u^2 \ell_n(W^\star,u^\star)\mid \calH_n] = \overline{\calI}_n^u - \overline{\calC}_n^u$, establishing \cref{eq:target-posterior-fisher}.

\medskip
\noindent To proceed, consider the filtration that reveals
$(W^\star,u^\star,\calH_{t-1},a_t)$ immediately before $Z_t$ is drawn.  Relative to this
filtration, $B_t=\eps p_t(1-p_t)\aone_t{\aone_t}^\top$ is predictable and $Z_t$ is conditionally standard Gaussian, so that for every $\eta>0$, 
\begin{equation*}
    \bbE\left[e^{-\eta Z_tB_t} \, \middle| \, W^\star,u^\star,\calH_{t-1},a_t\right]
    =e^{\eta^2B_t^2/2}. 
\end{equation*}
We use the standard matrix Laplace-transform
supermartingale inequality \citep{tropp2011freedman}: if $(X_t)_{t=1}^n$ is an adapted sequence of Hermitian matrices and $(A_t)_{t=1}^n$ is predictable with $\bbE_{t-1}[e^{X_t}]\preceq e^{A_t}$, then
\begin{equation*}
    \bbE \big[ \Tr \big( e^{\sum_{t=1}^nX_t-\sum_{t=1}^nA_t} \big) \big]
    \le d.
\end{equation*}
Applying this with $X_t=-\eta Z_tB_t$ and $A_t=\eta^2B_t^2/2$ and noting that $\calC_n^u=\sum_{t=1}^nZ_tB_t$, gives
\begin{equation*}
    \bbE \Big[
      \Tr \Big( e^{ -\eta\mathcal C_n^u -\frac{\eta^2}{2}\sum_{t=1}^nB_t^2} \Big)
    \Big]
    \le d.
\end{equation*}
Because $\sum_{t=1}^nB_t^2\preceq\sum_{t=1}^n \eps^2 p_t^2\aone_t{\aone_t}^\top=\calI_n^u$ and the trace exponential is
monotone in the Loewner order, we obtain
\begin{equation*}
    \bbE \Big[ \Tr \Big( e^{
      -\eta\mathcal C_n^u-\frac{\eta^2}{2}\calI_n^u} \Big)\Big]
    \le d.
\end{equation*}
The map $M\mapsto\Tr(e^M)$ is convex, and
$-\eta\overline{\mathcal C}_n^u-\frac{\eta^2}{2}\overline{\calI}_n^u$
is the conditional expectation of
$-\eta\mathcal C_n^u-\frac{\eta^2}{2}\calI_n^u$ given
$\calH_n$. Jensen's inequality then gives
\begin{equation*}
\begin{aligned}
    &\bbE \Big[ \Tr \Big(e^{ -\eta\overline{\mathcal C}_n^u - \frac{\eta^2}{2}\overline{\calI}_n^u} \Big) \Big] \le
    \bbE \Big[ \Tr \Big( e^{
        -\eta\mathcal C_n^u
        -\frac{\eta^2}{2}\calI_n^u}
      \Big)
    \Big]
    \le d. 
\end{aligned}
\end{equation*}
Markov's inequality now shows that, with
probability at least $1-\delta$,
\begin{equation*}
    \Tr \Big( e^{-\eta\overline{\mathcal C}_n^u
      -\frac{\eta^2}{2}\overline{\calI}_n^u} \Big)
    \le\frac{d}{\delta} \Longrightarrow -\overline{\mathcal C}_n^u \preceq \frac{\eta}{2}\overline{\calI}_n^u + \frac{\log(d/\delta)}{\eta}. 
\end{equation*}
Choosing $\delta = \frac{1}{4}, \eta = 1$ and combining with \cref{eq:target-posterior-fisher}, we arrive at the desired lemma. 
\end{proof}
\medskip
\noindent Let $\calS_n$ be the intersection of $\Tr(\Cov(u^\star \mid \calH_n))\le 4c_1$ and the event in \Cref{lem:likelihood-Fisher}, then $\bbP(\calS_n)\ge \frac{1}{2}$. On $\calS_n$, we have
\begin{align*}
\Tr\left((dI_d + \overline{\calI}_n^u)^{-1}\right) \le 4Cc_1. 
\end{align*}
When $c_1$ is sufficiently small, this implies that at least $d/2$ eigenvalues of $\overline{\calI}_n^u$ must be at least $d$, which is the statement of \Cref{lem:posterior-target-spread}.

\subsection{Proof of \Cref{lem:Kt-property}}
In this section we verify the properties of $K_t$ in \Cref{lem:Kt-property}. 

\paragraph{Submartingale.} By the Bayes rule, the posterior density $\rho_{t,W}$ of $W$ is a martingale $\rho_{t,W} = \bbE[\rho_{t+1,W}\mid \calH_t]$, averaged over the randomness in $r_{t+1}$. Since the Fisher information $J(\cdot)$ is convex, we have
\begin{align*}
J(\rho_{t,W}) = J\left(\bbE[\rho_{t+1,W}\mid \calH_t] \right) \preceq \bbE[J(\rho_{t+1,W})\mid \calH_t]. 
\end{align*}
This implies the submartingale property of $(K_t)_{t\ge 0}$. 

\paragraph{Expected trace.} We claim that
\begin{align}\label{eq:claim-trace}
\bbE[\Tr(K_n)] = d + \frac{1}{d^2}\bbE[ \| \nabla_{W^\star} \log p(\calH_n \mid W^\star) \|_F^2 ], 
\end{align}
where $p(\calH_n \mid W^\star)$ is the conditional density of $\calH_n$ given only $W^\star$, marginalized over $u^\star$. To see it, let
\begin{align*}
S_0(W) = \nabla \log \rho_{0,W}(W), \quad S_n(W,\calH_n) = \nabla_W \log p(\calH_n \mid W) 
\end{align*}
be the score of the prior and likelihood, respectively. By Bayes' rule, we have $\nabla \log \rho_{n,W}(W^\star) = S_0(W^\star) + S_n(W^\star, \calH_n)$, and hence
\begin{align*}
\bbE[\Tr(K_n)] &= \frac{1}{d^2}\bbE[\|\nabla \log \rho_{t,W}(W^\star)\|_F^2] \\
&= \frac{1}{d^2}\bbE[\|S_0(W^\star) + S_n(W^\star,\calH_n)\|_F^2]  \\
&\overset{(a)}{=} \frac{1}{d^2}\bbE[\|S_0(W^\star)\|_F^2] + \frac{1}{d^2}\bbE[\|S_n(W^\star,\calH_n)\|_F^2] \\
&\overset{(b)}{=} d + \frac{1}{d^2}\bbE[\|S_n(W^\star,\calH_n)\|_F^2], 
\end{align*}
where (a) follows from the score equation $\bbE[S_n(W^\star,\calH_n)\mid W^\star]=0$, and (b) follows from $S_0(W^\star) = -dW^\star$ by the Gaussian prior. Therefore \cref{eq:claim-trace} holds. 

\medskip
\noindent To further upper bound $\bbE[\|S_n(W^\star,\calH_n)\|_F^2]$, note that 
\begin{align*}
S_n(W^\star, \calH_n) & = \nabla_{W^\star} \log p(\calH_n \mid W^\star) \\
&= \frac{1}{p(\calH_n \mid W^\star)} \int \nabla_{W^\star} p(\calH_n\mid W^\star, u^\star) p(u^\star\mid W^\star) \mathrm{d}u^\star\\
&= \frac{1}{p(\calH_n \mid W^\star)} \int \nabla_{W^\star} \log p(\calH_n\mid W^\star, u^\star) p(\calH_n, u^\star\mid W^\star) \mathrm{d}u^\star\\
&= \int p(u^\star\mid \calH_n, W^\star) \nabla_{W^\star} \log p(\calH_n \mid W^\star, u^\star)\mathrm{d}u^\star \\
&= \bbE[\nabla_{W^\star} \log p(\calH_n \mid W^\star, u^\star) \mid \calH_n, W^\star]. 
\end{align*}
By Jensen's inequality, 
\begin{align*}
\bbE[\|S_n(W^\star,\calH_n)\|_F^2] &\le \bbE[\|\nabla_{W^\star} \log p(\calH_n \mid W^\star, u^\star)\|_F^2] \\
&=\bbE\Big[\Big\|\sum_{t=1}^n Z_t \nabla_{W^\star} f_{\theta^\star}(a_t) \Big\|_F^2 \Big] = \sum_{t=1}^n \bbE[\|\nabla_{W^\star} f_{\theta^\star}(a_t)\|_F^2] \le \frac{n\varepsilon^2}{r^2}, 
\end{align*}
where the last step follows from the deterministic inequality
\begin{align*}
\|\nabla_{W^\star} f_{\theta^\star}(a_t)\|_F^2 = \sum_{i,j=1}^d |\nabla_{W_{ij}^\star} f_{\theta^\star}(a_t)|^2 &= \frac{\varepsilon^2}{r^2} (1-p_t)^2  \sum_{i,j=1}^d\frac{|(W^\star a_t^1 - \frac{r}{8\varepsilon}a_t^2)_i (a_t^1)_j|^2}{\|W^\star a_t^1 - \frac{r}{8\varepsilon}a_t^2\|^2} \\
&= \frac{\varepsilon^2}{r^2} (1-p_t)^2  \|a_t^1\|^2 \le \frac{\varepsilon^2}{r^2}.
\end{align*}
Now the target inequality \cref{eq:info-constraint} follows from \cref{eq:claim-trace}. 

\paragraph{Anti-concentration.} To prove the anti-concentration result in \cref{eq:anti-concentration}, we need the following lemma. 

\begin{lemma}\label{lem:anti-concentration}
Let $Z\sim f$ be a random vector in $\bbR^d$ with $d\ge 3$ and a finite Fisher information matrix $J(f)=\bbE[\nabla\log f(Z)\nabla\log f(Z)^\top]$. Then there is an absolute constant $C>0$ such that for every $s>0$, 
\begin{align*}
    \sup_{z\in \bbR^d} \bbP(\|Z-z\|\le s) \le \frac{Cs^2}{d^2}\Tr(J(f)). 
\end{align*}
\end{lemma}
\begin{proof}
Since the Fisher information is invariant with translations, it suffices to consider $z=0$. By Hölder's inequality, 
\begin{align*}
\bbP(\|Z\|\le s) = \int_{s\bbB^d} f(z)\mathrm{d}z \le \mathrm{vol}_d(s\bbB^d)^{\frac{2}{d}} \left( \int_{\bbR^d} f(z)^{\frac{d}{d-2}}\mathrm{d}z \right)^{\frac{d-2}{d}}. 
\end{align*}
Since $\mathrm{vol}_d(\bbB^d)^{1/d} = \sqrt{\pi} \Gamma(\frac{d}{2}+1)^{-1/d} = O(d^{-1/2})$, the first term scales as $\mathrm{vol}_d(s\bbB^d)^{\frac{2}{d}} = O(\frac{s^2}{d})$. For the second term, the sharp Euclidean Sobolev inequality in \cite{talenti1976} gives
\begin{align*}
\int_{\bbR^d} f(z)^{\frac{d}{d-2}}\mathrm{d}z  = \int_{\bbR^d} (\sqrt{f(z)})^{\frac{2d}{d-2}}\mathrm{d}z \le \left(C_d^2 \int_{\bbR^d} \|\nabla \sqrt{f(z)}\|^2 \mathrm{d}z\right)^{\frac{d}{d-2}}, 
\end{align*}
with the sharp constant $C_d$ given by
\begin{align*}
C_d = \frac{1}{\sqrt{\pi d(d-2)}}\left(\frac{\Gamma(d)}{\Gamma(\frac{d}{2})}\right)^{\frac{1}{d}} = O(d^{-1/2}). 
\end{align*}
A combination of the above displays gives
\begin{align*}
\bbP(\|Z\|\le s) \le \frac{Cs^2}{d^2} \int_{\bbR^d} \|\nabla \sqrt{f(z)}\|^2 \mathrm{d}z = \frac{Cs^2}{4d^2} \Tr(J(f)),
\end{align*}
as claimed. 
\end{proof}

\noindent We apply \Cref{lem:anti-concentration} to $Z=W^\star a_t^1$ and $z=\frac{r}{8\varepsilon} a_t^2$, with $W^\star \sim \rho_{t-1,W}$ following the posterior distribution of $W^\star$ given $\calH_{t-1}$. We claim that
\begin{align}\label{eq:Fisher-claim}
\frac{\Tr(J(Z))}{d^2}\le \frac{1}{{\aone_t}^\top K_{t-1}^{-1} a_t^1}. 
\end{align}
To show \cref{eq:Fisher-claim}, for notational simplicity, write $\rho = \rho_{t-1,W}$, $a=a_t^1$, and $W=W^\star$. We first show that the row score matrix $S_{W}(W) \in \bbR^{d\times d}$ with rows $(\nabla_{W_{i,:}}\log \rho(W)^\top)_{i\in [d]}$ and the vector score $S_{Z}(Z) = \nabla \log f_Z(Z)\in \bbR^d$ are related via the identity $\bbE[S_W(W)\mid Z] = S_Z(Z)a^\top$. To see it, note that for any smooth test function $\phi: \mathbb{R}^d \to \mathbb{R}$,
\begin{align*}
\bbE[\phi(Z)S_Z(Z)a^\top] = -\bbE[\nabla_Z \phi(Z)a^\top] = -\bbE[\nabla_W \phi(Wa)] = \bbE[\phi(Wa)S_W(W)]. 
\end{align*}
Next, since $\bbE[\nabla_{W_{i,:}}\log \rho(W)\mid Z] = S_{Z,i}(Z)a$ for every $i\in [d]$, conditional Jensen implies
\begin{align*}
\bbE[S_{Z,i}(Z)^2] a a^\top &= \bbE\left[\bbE[\nabla_{W_{i,:}}\log \rho(W)\mid Z] \bbE[\nabla_{W_{i,:}}\log \rho(W) \mid Z]^\top \right] \\
&\preceq \bbE[ \nabla_{W_{i,:}}\log \rho(W)\nabla_{W_{i,:}}\log \rho(W)^\top ]. 
\end{align*}
Summing over $i=1,\dots,d$ gives
\begin{align*}
\frac{\Tr(J(Z))}{d^2} aa^\top \preceq K(\rho) &\Longleftrightarrow \frac{\Tr(J(Z))}{d^2} (K(\rho)^{-1/2}a)(K(\rho)^{-1/2}a)^\top \preceq I_d \\
&\Longleftrightarrow \frac{\Tr(J(Z))}{d^2} \|K(\rho)^{-1/2}a\|^2 \le 1, 
\end{align*}
by definition of $K(\rho)$ in \cref{eq:Kt}. Since $K_{t-1}=K(\rho_{t-1,W})$, this gives \cref{eq:Fisher-claim}. Applying \Cref{lem:anti-concentration} now yields
\begin{align}\label{eq:basic-ineq}
\bbP(f_t^1 \le C) {\aone_t}^\top K_{t-1}^{-1} a_t^1 = \bbP(\|W^\star a_t^1 - \frac{r}{8\varepsilon}a_t^2\| \le Cr) {\aone_t}^\top K_{t-1}^{-1} a_t^1 \lesssim C^2 r^2. 
\end{align}

\noindent We now deduce the target from \cref{eq:basic-ineq}. Let $h_t={\aone_t}^\top K_{t-1}^{-1}\aone_t$, and $\calE=\{\norm{u^\star}\le1/2\}$. Under the event $\calE$, we have $f_t^2\le1$, and hence $p_t\le\exp(-(f_t^1-1)_+)$. Carrying out a dyadic decomposition over $f_t^1$ and applying \cref{eq:basic-ineq} at radii $C=2,4,8,\ldots$ gives
\begin{align*}
    h_t\bbE\left[
      p_t^2\Ind_{\calE}
      \,\middle|\,\calH_{t-1},a_t
    \right]
    &\lesssim
      h_t\bbP(f_t^1\le 2\mid\calH_{t-1},a_t)
      +\sum_{j\ge 1}e^{-c2^j}h_t
        \bbP(f_t^1\le 2^{j+1}\mid\calH_{t-1},a_t) \nonumber\\
    &\lesssim
      r^2+r^2\sum_{j\ge 1}4^{j+1}e^{-c2^j}
    \lesssim r^2.
\end{align*}
This is the desired bound in \cref{eq:anti-concentration}. 

\subsection{Proof of \Cref{lem:tension-estimation}}
We construct the orthogonal projection matrix $Q_0,\dots,Q_n$ recursively as follows. Initialize $Q_0 = I_d$. Suppose that $Q_{t-1}$ has been constructed at the end of time $t-1$. At time $t$, after $K_t$ is updated, perform the eigendecomposition of $Q_{t-1}K_tQ_{t-1}$: 
\begin{align*}
Q_{t-1}K_tQ_{t-1} = \sum_{j=1}^d \lambda_{t,j} u_{t,j} u_{t,j}^\top. 
\end{align*}
Clearly, every eigenvector $u_{t,j}$ with $\lambda_{t,j}>0$ must belong to the eigenspace associated with eigenvalue $1$ of the matrix $Q_{t-1}$. For some $\Lambda>0$ to be determined later, define
\begin{align}\label{eq:Qt}
Q_t = \sum_{j: 0<\lambda_{t,j}\le \Lambda} u_{t,j}u_{t,j}^\top 
\end{align}
be the orthogonal projection such that all eigendirections with eigenvalue larger than $\Lambda$ are removed from $Q_{t-1}$. It is then clear that $0\preceq Q_t \preceq Q_{t-1}$, $Q_{t-1}Q_t=Q_tQ_{t-1}=Q_t$, and $Q_tK_tQ_t \preceq \Lambda Q_t$. Also $Q_n$ is $\calH_n$-measurable. 

\medskip
\noindent We first choose $\Lambda$ to ensure that $\bbE[\Tr(I_d-Q_n)]\le \frac{d}{16}$. Since $K_t$ has eigenvalues at least $\Lambda$ on all directions of $Q_{t-1}-Q_t$, we have 
\begin{align*}
\Lambda \Tr(Q_{t-1}-Q_t) \le \Tr(K_t(Q_{t-1}-Q_t)). 
\end{align*}
Taking expectation on both sides, and using the submartingale property of $K_t\preceq \bbE[K_n\mid \calH_t]$ in \Cref{lem:Kt-property}, we get
\begin{align*}
\bbE[\Tr(Q_{t-1})] - \bbE[\Tr(Q_t)] \le \frac{1}{\Lambda}\bbE[\Tr(K_n(Q_{t-1}-Q_t))]. 
\end{align*}
Summing over $t=1,\dots,n$ and telescoping now gives
\begin{align*}
\bbE[\Tr(I_d-Q_n)] \le \frac{1}{\Lambda}\bbE[\Tr(K_n(I_d-Q_n))] \le \frac{1}{\Lambda}\bbE[\Tr(K_n)] \le \frac{1}{\Lambda}\left(d+ \frac{\varepsilon^2 n}{r^2d^2} \right), 
\end{align*}
by \Cref{lem:Kt-property}. Choosing
\begin{align}\label{eq:Lambda-estimation}
\Lambda = 16\left(1+\frac{\eps^2 n}{r^2d^3}\right), 
\end{align}
this gives $\bbE[\Tr(I_d-Q_n)]\le \frac{d}{16}$. 

\medskip
\noindent For the second inequality \cref{eq:tension}, note that $Q_tK_tQ_t\preceq \Lambda Q_t$ implies that $K_t^{-1} \succeq \frac{1}{\Lambda}Q_t$. By \cref{eq:anti-concentration} in \Cref{lem:Kt-property}, we have
\begin{align*}
    Cr^2 &\ge {\aone_t}^\top K_{t-1}^{-1} \aone_t \cdot \bbE[p_t^2\Ind_{\|u^\star\|\le 1/2}\mid \calH_{t-1},a_t] \\
    &\ge \frac{1}{\Lambda} \bbE[{\aone_t}^\top Q_{t-1} \aone_t\cdot p_t^2\Ind_{\|u^\star\|\le 1/2}\mid \calH_{t-1},a_t] \\
    &\ge \frac{1}{\Lambda} \bbE[{\aone_t}^\top Q_n \aone_t\cdot p_t^2\Ind_{\|u^\star\|\le 1/2}\mid \calH_{t-1},a_t],
\end{align*}
where we have used $Q_{t-1}\succeq Q_n$. Summing over $t\in [n]$ and using \cref{eq:target-likelihood-fisher} gives
\begin{align*}
    \bbE[\Tr(\calI_n^u Q_n)\Ind_{\|u^\star\|\le 1/2}]\le C\Lambda\varepsilon^2 r^2 n.  
\end{align*}
Therefore,
\begin{align*}
\bbE[\Tr(\overline{\calI}_n^u Q_n)] &= \bbE[\Tr(\calI_n^u Q_n)] \\
&= \bbE[\Tr(\calI_n^u Q_n)\Ind_{\|u^\star\|\le 1/2}] + \bbE[\Tr(\calI_n^u Q_n)\Ind_{\|u^\star\|> 1/2}] \\
&\lesssim \varepsilon^2 (\Lambda r^2 + e^{-cd})n, 
\end{align*}
where the last step follows from $\bbP(\|u^\star\|>\frac{1}{2})\le e^{-cd}$ and the deterministic inequality $\Tr({\calI}_n^u Q_n)\le \Tr({\calI}_n^u) \le n\varepsilon^2$. Plugging the expression of $\Lambda$ into \cref{eq:Lambda-estimation} gives \cref{eq:tension}. 

\subsection{Proof of \Cref{lem:tension-regret}}
The proof is similar to \Cref{lem:tension-estimation}, with the same construction of $Q_t$ but a different choice of the parameter $\Lambda$. Our proof will use an auxiliary matrix $\widetilde{K}_t$, whose construction is motivated by the following lemma. 

\begin{lemma}\label{lem:Kt-upper-bound}
Let $S_t(W) = \nabla_W \log \frac{\rho_{t,W}(W)}{\rho_{0,W}(W)}$ be the relative score, viewed as a $d\times d$ matrix. Then
\begin{align*}
K_t \preceq 2I_d + \frac{1}{d^2}\bbE[ S_t(W^\star)^\top S_t(W^\star) \mid \calH_t ]. 
\end{align*}
\end{lemma}
\begin{proof}
Let $w=\mathrm{vec}(W^\star) \in \bbR^{d^2}$ and $s=\mathrm{vec}(S_t(W^\star)) \in \bbR^{d^2}$. Since $\nabla_{W^\star} \log \rho_{0,W}(W^\star) = -dw$, we have $\nabla_{W^\star} \log \rho_{t,W}(W^\star) = -dw + s$. For any $v\in \bbR^{d^2}$, standard score identity gives
\begin{align*}
-\|v\|^2 &= \bbE[(v^\top w)(v^\top \nabla_{W^\star} \log \rho_{t,W}(W^\star))] \\
&= \bbE[(v^\top w)(v^\top (-dw + s) )] \\
&= -d\bbE[(v^\top w)^2] + \bbE[(v^\top w)(v^\top s)], 
\end{align*}
which implies that
\begin{align*}
d \bbE[ww^\top] = I_{d^2} + \frac{1}{2}\bbE[ws^\top + sw^\top]. 
\end{align*}
Therefore, the Fisher information of $\rho_{t,W}$ can be written as
\begin{align*}
J(\rho_{t,W}) &= \bbE[(-dw+s)(-dw+s)^\top] \\
&= d^2 \bbE[ww^\top] - d\bbE[ws^\top+sw^\top] + \bbE[ss^\top] \\
&= 2dI_{d^2} - d^2\bbE[ww^\top] + \bbE[ss^\top] \\
&\preceq 2dI_{d^2} +  \bbE[ss^\top]. 
\end{align*}
Taking the normalized output partial trace gives the claim. 
\end{proof}

\medskip
\noindent By simple algebra, one can show that 
\begin{align*}
S_t(W^\star) = \sum_{s=1}^t \bbE[Z_s \nabla_{W^\star} f_{\theta^\star}(a_s) \mid \calH_t, W^\star].
\end{align*}
Now we define
\begin{align}\label{eq:St-tilde}
\widetilde{S}_t(W^\star) = \sum_{s=1}^t \bbE[Z_s \nabla_{W^\star} f_{\theta^\star}(a_s) \mid \calH_t, W^\star] Q_{s-1}, 
\end{align}
and
\begin{align}\label{eq:Kt-tilde}
\widetilde{K}_t = 2I_d + \frac{1}{d^2}\bbE[ \widetilde{S}_t(W^\star)^\top \widetilde{S}_t(W^\star) \mid \calH_t ]. 
\end{align}
Intuitively, $\widetilde{S}_t$ is the  cumulative relative score along the directions which are still unlearned, and $\widetilde{K}_t$ is an upper bound of the information matrix relevant to the unlearned directions. 

\medskip
\noindent We establish three properties of $\widetilde{K}_t$: 
\begin{enumerate}
    \item $(\widetilde{K}_t)_{t\ge 0}$ is a submartingale. In fact, since $\bbE[Z_{t+1}\mid \calH_t, W^\star]=0$, we have $\bbE[\widetilde{S}_{t+1}(W^\star)\mid \calH_t, W^\star] = \widetilde{S}_t(W^\star)$ by \cref{eq:St-tilde}. Therefore, $\bbE[\widetilde{K}_{t+1}\mid \calH_t] \succeq \widetilde{K}_t$ by conditional Jensen. 
    \item For every $t\ge 1$, $Q_{t-1}K_tQ_{t-1}\preceq Q_{t-1}\widetilde{K}_tQ_{t-1}$. Because $Q_{s-1}Q_{t-1}=Q_{t-1}$ for all $s\le t$, we have $\widetilde{S}_t(W^\star)Q_{t-1} = S_t(W^\star)Q_{t-1}$, and the conclusion follows from \Cref{lem:Kt-upper-bound} and \cref{eq:Kt-tilde}. 
    \item The trace of $\widetilde{K}_n$ satisfies
\begin{align}\label{eq:Kt-tilde-trace}
\bbE[\Tr(\widetilde{K}_n)] \le 2d+\frac{\varepsilon^2}{r^2 d^2}  \sum_{t=1}^n \bbE\|Q_{t-1} \aone_t\|^2. 
\end{align}
Indeed, conditional Jensen's inequality gives
\begin{align*}
\bbE[\Tr(\widetilde{K}_n)] &\le 2d + \frac{1}{d^2} \bbE \left\| \sum_{t=1}^n Z_t\nabla_{W^\star}f_{\theta^\star}(a_t) Q_{t-1} \right\|_F^2 \\
&= 2d + \frac{1}{d^2}\sum_{t=1}^n \bbE\| \nabla_{W^\star}f_{\theta^\star}(a_t) Q_{t-1} \|_F^2 \\
&= 2d + \frac{\eps^2}{r^2d^2}\sum_{t=1}^n \bbE[(1-p_t)^2\|Q_{t-1}a_t^1\|^2], 
\end{align*}
where the last step follows from the computation of $\nabla_{W^\star} f_{\theta^\star}(a_t)$ in the proof of \Cref{lem:Kt-property}. This implies \cref{eq:Kt-tilde-trace}. 
\end{enumerate}
\noindent To proceed, note that the second property implies
\begin{align*}
\Tr(K_t(Q_{t-1}-Q_t)) &= \Tr((Q_{t-1}-Q_t)K_t(Q_{t-1}-Q_t)) \\
&= \Tr((Q_{t-1}-Q_t)Q_{t-1}K_tQ_{t-1}(Q_{t-1}-Q_t)) \\
&\le \Tr((Q_{t-1}-Q_t)Q_{t-1}\widetilde{K}_tQ_{t-1}(Q_{t-1}-Q_t)) \\
&= \Tr(\widetilde{K}_t(Q_{t-1}-Q_t)), 
\end{align*}
and thus the proof of \Cref{lem:tension-estimation} gives
\begin{align*}
\Lambda \Tr(Q_{t-1}-Q_t) \le \Tr(K_t(Q_{t-1}-Q_t)) \le \Tr(\widetilde{K}_t(Q_{t-1}-Q_t)). 
\end{align*}
Using the same telescoping and the submartingale property of $(\widetilde{K}_t)_{t\ge 0}$, we arrive at 
\begin{align}\label{eq:codim-from-trace}
\bbE[\Tr(I_d-Q_n)] \le \frac{1}{\Lambda}\bbE[\Tr(\widetilde{K}_n)] \le \frac{1}{\Lambda}\left(2d + \frac{\varepsilon^2}{r^2 d^2}  \sum_{t=1}^n \bbE[\|Q_{t-1} \aone_t\|^2] \right). 
\end{align}

\medskip
\noindent To relate this quantity to the regret, by \cref{eq:target-excess-mse}: 
\begin{align*}
f_{\theta^\star}(a) - f_{\theta^\star}(a_{\theta^\star}) \ge \frac{c\varepsilon}{r}\|W^\star a^1 - \frac{r}{8\varepsilon}a^2\| =: \frac{c\varepsilon}{r}D(a). 
\end{align*}
Next, for every $s>0$, \cref{eq:basic-ineq} gives
\begin{align*}
\bbP(D(a_t)\le s \mid \calH_{t-1}, a_t) {\aone_t}^\top K_{t-1}^{-1} \aone_t \le C s^2. 
\end{align*}
Choosing $s\asymp ({\aone_t}^\top K_{t-1}^{-1} \aone_t)^{1/2}$ and applying Markov's inequality, this gives $$\bbE[D(a_t)\mid \calH_{t-1}, a_t] \ge c'({\aone_t}^\top K_{t-1}^{-1} \aone_t)^{1/2} \ge \frac{c'}{\sqrt{\Lambda}}\|Q_{t-1}\aone_t\|$$ for absolute constants $c'>0$, using $K_{t-1}^{-1}\succeq \frac{1}{\Lambda}Q_{t-1}$. As a result, 
\begin{align*}
\mathfrak{R}_n \ge \frac{c\varepsilon}{r}\sum_{t=1}^n \bbE[D(a_t)] \ge \frac{cc'\varepsilon}{\sqrt{\Lambda} r} \sum_{t=1}^n \bbE[\|Q_{t-1}\aone_t\|] \ge \frac{cc'\varepsilon}{\sqrt{\Lambda} r} \sum_{t=1}^n \bbE[\|Q_{t-1}\aone_t\|^2],
\end{align*}
which together with \cref{eq:codim-from-trace} gives
\begin{align*}
\bbE[\Tr(I_d-Q_n)] \le \frac{1}{\Lambda}\bbE[\Tr(\widetilde{K}_n)] \le C\left(\frac{d}{\Lambda}+ \frac{\eps \mathfrak{R}_n}{rd^2\sqrt{\Lambda}}\right). 
\end{align*}
Finally, we choose 
\begin{align*}
\Lambda = C'\left( 1 + \frac{\eps^2 \mathfrak{R}_n^2}{r^2 d^6} \right)
\end{align*}
with a large enough constant $C'>0$ to ensure that $\bbE[\Tr(I_d-Q_n)]\le \frac{d}{16}$. Then \cref{eq:tension-2} follows from the same analysis in the proof of \Cref{lem:tension-estimation} and the new choice of $\Lambda$ above.

\section{Discussion}
\label{sec:discussion}

\noindent This section examines the scope of the lower bound.  We first show a matching algorithm achieving an $\calOtilde ( \sqrt{d} n^{3/4} \wedge d^{4/3} \sqrt{n})$ regret upper bound when noisy observations are generated from a function belonging to the support of the bounded family considered in \Cref{thm:gaussian-hidden-map-estimation}. We then extend the lower bounds in \Cref{thm:convex-optimization-lower-bound,thm:convex-bandit-regret-lower-bound} to unconstrained action spaces. We finally discuss some extensions that are unlikely to succeed in improving the construction.

\input{algorithm_new}

\subsection{Extension to Unconstrained Action Spaces}
\label{sec:unconstrained-action-spaces}

The lower bounds in
\Cref{thm:convex-optimization-lower-bound,thm:convex-bandit-regret-lower-bound}
were proved for the Euclidean ball.  We now show that the same bounds hold
when the learner may query any point in $\bbR^d$.  The main difficulty is that
a direct extension of the hard functions could make distant queries highly
informative.  We avoid this by joining each hard function to a fixed conic
branch outside a neighborhood of the unit ball.

\begin{theorem}[Unconstrained Estimation and Regret Lower Bounds]
\label{thm:unconstrained-convex-bandit-regret-lower-bound}
There are universal constants $c_0,c_1,c,C>0$ such that the following holds
for all sufficiently large $d$.  For every $\eps\in(0,1/2]$ and every
adaptive learner that makes
\begin{equation}
    n<c_1\min\left\{
      \frac{d^{5/2}}{\eps^2},\frac{d^2}{\eps^4}
    \right\}
    \label{eq:unconstrained-query-budget}
\end{equation}
noisy function-value queries and outputs an $\calH_n$-measurable
$\widehat a\in\bbR^d$, there exists a function
$f\in\calF_d(\bbR^d)$ such that
\begin{equation*}
    \bbE_f\left[
      f(\widehat a)-\min_{a\in\bbR^d}f(a)
    \right]
    \ge c_0\eps.
\end{equation*}
Moreover, for every $d\ge C$ and $n\ge Cd^2$,
\begin{equation}
    \mathfrak R_n^\star(d;\bbR^d)
    \ge c\min\left\{
      \sqrt d\,n^{3/4},d^{4/3}\sqrt n
    \right\}.
    \label{eq:unconstrained-convex-bandit-regret-lower-bound}
\end{equation}
\end{theorem}

\medskip
\noindent The formal proof is deferred to
\Cref{sec:conic-extension-proofs}.  We describe here the construction and why
the information bottlenecks from Sections~2 and~3 continue to apply.

\paragraph{Smoothed conic extension.}
We work in dimension $2d$.  Let $\theta=(W,u)$ belong to the support of
the bounded prior $\pi_{\mathrm b}$, and let $f_\theta$ be the hard function
in \cref{eq:intro-hidden-map-softmax}.  The parameters below satisfy
$\eps\le1/2$ and $\eps/r\le1/16$.  By
\Cref{lem:target-excess-controls-estimation,lem:high-probability-lipschitz},
$f_\theta$ is convex and $1$-Lipschitz on $\bbB^{2d}$, satisfies
$f_\theta(0)=0$, and has minimizer
\begin{equation*}
    a_\theta
    =\left(u,\frac{8\eps}{r}Wu\right),
    \qquad
    \norm{a_\theta}\le\frac{3\sqrt5}{8}<1.
\end{equation*}
Define the radial projection
\begin{equation*}
    \Pi(z)\triangleq\frac{z}{1\vee\norm z}.
\end{equation*}
For $a=(\aone,\atwo)$, write
$\Pi(a)=(\Pi^1(a),\Pi^2(a))$, where
$\Pi^j(a)=a^j/(1\vee\norm a)$ for $j\in\{1,2\}$.

\medskip
\noindent Fix $\lambda=1/64$, and let $\Psi$ be the convex smoothed hinge in
\cref{eq:smoothed-hinge-definition}.  Define
\begin{align}
    \widetilde f_\theta(a)
      &\triangleq\inf_{z\in\bbB^{2d}}
        \{f_\theta(z)+\norm{a-z}\},\nonumber\\
    \Gamma(a)
      &\triangleq\lambda+\frac12(\norm a-1),\nonumber\\
    g_\theta(a)
      &\triangleq\Gamma(a)+
        \Psi\left(
          \lambda\widetilde f_\theta(a)-\Gamma(a)-2\lambda
        \right).
    \label{eq:unconstrained-hard-function}
\end{align}
The infimal-convolution extension $\widetilde f_\theta$ agrees with
$f_\theta$ on the unit ball and remains convex and $1$-Lipschitz on
$\bbR^{2d}$.  The fixed branch $\Gamma$ prevents queries far from the ball
from revealing the parameter, while $\Psi$ joins the two branches without
destroying convexity or Lipschitzness.

\begin{lemma}[Conic Extension]
\label{lem:conic-extension}
There is a universal $\mathfrak r_0<1$ such that, for every $\theta$ in the
support of $\pi_{\mathrm b}$, the function
$g_\theta$ belongs to $\calF_{2d}(\bbR^{2d})$, and $a_\theta$ is a global
minimizer of $g_\theta$.  Moreover,
$g_\theta(a)=\Gamma(a)$ whenever $\norm a\ge\mathfrak r_0$, and, for every
$a\in\bbR^{2d}$,
\begin{equation}
    g_\theta(a)-g_\theta(a_\theta)
    \ge
    \lambda\bigl(
      f_\theta(\Pi(a))-f_\theta(a_\theta)
    \bigr)
    \ge
    c\lambda\eps\norm{\Pi^1(a)-u}^2.
    \label{eq:conic-regret-domination}
\end{equation}
\end{lemma}

\medskip
\noindent Thus every query with $\norm a\ge\mathfrak r_0$ has a mean loss that is
independent of $\theta$, while every nearly optimal action must have
$\Pi^1(a)$ close to $u$. 

\paragraph{Auxiliary observation model.}
The bounded prior is needed to guarantee Lipschitzness, whereas the
information arguments in Sections~2 and~3 use the unconditioned Gaussian
prior $\pi$.  To connect the two, define the auxiliary mean
\begin{equation}
    \mu_\theta(a)
    \triangleq
    \begin{cases}
      \Gamma(a)+
      \Psi\left(
        \lambda f_\theta(a)-\Gamma(a)-2\lambda
      \right),&\norm a<1,\\
      \Gamma(a),&\norm a\ge1.
    \end{cases}
    \label{eq:smoothed-hinge-gaussian-auxiliary-mean}
\end{equation}
On the bounded-prior event $\calE_{\mathrm b}$,
$\mu_\theta=g_\theta$ at every action.  Under the Gaussian prior,
$\mu_\theta$ is parameter independent outside the unit ball; hence all of
its parameter derivatives vanish there.  Inside the ball, the chain rule and
the derivative bounds for $\Psi$ preserve, up to universal factors, the
score, Hessian, and anti-concentration estimates used in Sections~2 and~3. Using similar techniques in \Cref{sec:hidden-map-family}, we first prove estimation complexity and regret lower bounds under the auxiliary observation model in \cref{eq:smoothed-hinge-gaussian-auxiliary-mean} and the Gaussian prior, and then move to bounded priors using the same ideas in \Cref{subsec:bounded-prior-reduction}. We defer the details to \Cref{sec:conic-extension-proofs}.

\subsection{Limitations of the Construction and Open Questions}
\label{sec:limitations-open-questions}

\Cref{thm:constructed-family-regret-upper-bound} shows that tuning the
tube width, $r$, of the present construction cannot improve its dimension dependence.
We next discuss why several direct extensions also do not immediately yield a
stronger lower bound. It is important to point out that some of these are based on heuristic arguments. For more than two arguments,
we use the notation
\begin{equation*}
  \smax(z_1,\ldots,z_m)
  \triangleq \log\bigg(\sum_{j=1}^m e^{z_j}\bigg)-\log(m).
\end{equation*}

\begin{enumerate}
  \item \emph{Changing the tube width.}  \Cref{thm:constructed-family-regret-upper-bound}
  applies for every admissible $r$ and achieves regret $\calOtilde ( \sqrt{d} n^{3/4} \wedge d^{4/3} \sqrt{n} )$.  Thus changing $r$ alone cannot produce
  a larger regret lower bound for the prior considered.

  \item \emph{Several tubes sharing one output block.}  With two hidden maps,
  one might consider
  \begin{equation*}
    f_\theta(\aone,\atwo)
    =\eps\,\smax\left(
      \frac1r\left\|W_1^\star\aone-\frac{r}{8\eps}\atwo\right\|,
      \frac1r\left\|W_2^\star\aone-\frac{r}{8\eps}\atwo\right\|,
      \frac12\|\aone-u^\star\|^2-\frac12\|u^\star\|^2
    \right).
  \end{equation*}
  At $\aone=u^\star$, the same $\atwo$ must lie near both tube centers.  This
  requires $W_1^\star u^\star\approx W_2^\star u^\star$, which independent
  maps generally do not satisfy.  Thus the construction typically has no
  low-loss action near $\aone=u^\star$, and the intended estimation and regret reductions
  break down.

  \item \emph{Giving the tubes separate action blocks.}  The preceding
  conflict disappears if each tube has its own output block.  The tubes may
  either branch from the same input or be chained serially.  As an example of
  the latter, for $a=(a^1,a^2,a^3)\in(\bbR^d)^3$ one could consider
  \begin{equation*}
    f_\theta(a^1,a^2,a^3)
    =\eps\,\smax\left(
      \frac1r\left\|W_1^\star a^1-\frac{r}{8\eps}a^2\right\|,
      \frac1r\left\|W_2^\star a^2-\frac{r}{8\eps}a^3\right\|,
      \frac12\|a^1-u^\star\|^2-\frac12\|u^\star\|^2
    \right).
  \end{equation*}
  In either layout, the learner need not guess the tubes simultaneously: it
  can learn and recenter one block at a time.  Thus the cost of learning the tubes grows additively rather than multiplicatively in the number of tubes. Since the action dimension also grows with the number of blocks, this modification therefore does not suggest a better exponent for $d$.
\end{enumerate}

\noindent These observations do not rule out other ways to extend this construction, e.g., by coupling the hidden tubes in a way that prevents such sequential recentering.  The best-known upper bounds have leading term $\calOtilde(d^{3/2}\sqrt{n})$ in both the unconstrained and Euclidean-ball settings~\citep{lattimoregyorgy2023,fokkema2024}. We therefore make the following conjecture, which is a slight refinement of a conjecture posed by \cite{bubeckeldanlee2021}.

\begin{conjecture}[Minimax Regret]
\noindent The minimax expected regret of stochastic bandit convex optimization, in both the unconstrained setting $\calA = \bbR^d$, and the Euclidean ball action space $\calA = \bbB^d$ satisfies
\begin{equation*}
    \mathfrak R_n^\star(d ; \calA)
    =\widetilde{\Theta}(d^{3/2}\sqrt{n} \wedge \sqrt{d} n^{3/4} \wedge n).
\end{equation*}
\end{conjecture}

\subsection*{Acknowledgements}

NR would like to thank Dylan Foster for helpful comments and discussions over the course of writing this paper.

\bibliographystyle{plainnat}
\bibliography{references}

\appendix
\crefalias{section}{appendix}
\crefname{appendix}{appendix}{appendices}
\Crefname{appendix}{Appendix}{Appendices}
\phantomsection
\addcontentsline{toc}{section}{Appendix}
\addtocontents{toc}{\protect\setcounter{tocdepth}{0}}
\input{appendix}

\end{document}

%% file: commands.tex
\renewcommand{\Pr}{\mathrm{Pr}}

\usepackage{cases}

\newcommand{\Alg}{\mathsf{Alg}}
\newcommand{\calOtilde}{\widetilde{\mathcal{O}}}

\usepackage{xcolor}



\usepackage[most]{tcolorbox}
\tcbuselibrary{skins,breakable}

\newtcolorbox[
  auto counter,
  number within=section,
  number freestyle={\noexpand\thesection.\noexpand\arabic{\tcbcounter}~\noexpand\mytitle},
  crefname={Meta-algorithm}{Meta-algorithms}
]{metaalgorithm}[2][]{
  breakable,
  enhanced,
  colback=white,
  colframe=black,
  coltitle=black,
  fonttitle=\bfseries,
  code={\def\mytitle{#2}},     
  title=Meta-algorithm \thetcbcounter,
  boxrule=0.8pt,
  arc=2pt,
  left=8pt,right=8pt,
  top=6pt,bottom=6pt,
  attach boxed title to top center={yshift=-2mm},
  boxed title style={
    colback=white,
    colframe=black,
    boxrule=0.8pt,
    sharp corners
  },
  #1
}

\newcommand{\aone}{a^1}
\newcommand{\atwo}{a^2}
\newcommand{\Trout}{\operatorname{Tr}_{\mathrm{out}}}
\newcommand{\smax}{\operatorname{softmax}}
\newcommand{\Tr}{\operatorname{Tr}}
\usepackage{inconsolata}

%% file: algorithm_new.tex
\subsection{Tightness of the Regret Lower Bound for the Constructed Family}
\label{sec:tightness}

\noindent We next show that the regret lower bound is tight, up to logarithmic
factors, for the bounded family used in its proof.  The guarantee is
pointwise: the same policy works for every $(W^\star,u^\star)$ in the
support of $\pi_{\mathrm b}$. We assume that the learner knows the fixed problem
parameters $(r,\eps)$.

\begin{theorem}[Regret Upper Bound for the Constructed Family]
\label{thm:constructed-family-regret-upper-bound}
There is a universal constant $C>0$ such that the following holds for all
$d\ge C$, $n\ge Cd^2$, $\eps\in(0,1/2]$, and $r\ge 16\eps$.  There is a
learning algorithm, knowing $(r,\eps)$, such that, for every
$\theta^\star=(W^\star,u^\star)\in\operatorname{supp}(\pi_{\mathrm b})$,
\begin{equation}
  \bbE_{f_{\theta^\star}}\left[
    \sum_{t=1}^n
    \bigl(f_{\theta^\star}(a_t)
      -f_{\theta^\star}(a_{\theta^\star})\bigr)
  \right]
  \le
  C\log^3(edn)
  \min\left\{
    \sqrt d\,n^{3/4},d^{4/3}\sqrt n
  \right\}.
  \label{eq:constructed-family-regret-upper-bound}
\end{equation}
Here $a_{\theta^\star}$ is the minimizer in
\cref{eq:hidden-map-unconstrained-benchmark}.
\end{theorem}

\noindent Below we describe the learning algorithm. 

\medskip
\noindent\textit{Overall flow.}
Unless $\varepsilon$ is too small, the algorithm runs three
stages.  First, it optionally learns the images
$b_i=W^\star v_i$ of an orthonormal basis, but only to the accuracy $E$ needed
by the next stage.  Second, it uses signed queries along the $v_i$'s to
estimate the coordinates of $u^\star$.  Third, after obtaining
$\widehat u$, it freezes the first block at $\widehat u$ and tracks only
the single tube center $W^\star\widehat u$.  Thus the learner never needs
to estimate the entire matrix $W^\star$ accurately.

\medskip
\noindent The scale $\delta$ controls the signed queries in the second
stage.  A larger $\delta$ produces a stronger signal about $u^\star$, but
requires more accurate knowledge of the tube directions $W^\star v_i$.
Define
\begin{equation}
  M_n\triangleq
  \min\left\{\sqrt d\,n^{3/4},d^{4/3}\sqrt n\right\},
  \qquad
  \delta\triangleq
  \min\left\{1,
    \max\left\{r,\frac{\sqrt\eps\,n^{1/4}}d\right\}
  \right\},
  \qquad
  E\triangleq1\wedge\frac r\delta.
  \label{eq:tube-regret-probing-scale}
\end{equation}
When $\delta\le r$, one has $E=1$, so the deterministic estimate
$\widehat b_i=0$ is already accurate enough and the first stage is
skipped.  When $\delta>r$, the learner localizes every $b_i$ to error
$E=r/\delta$.  This ensures
\begin{equation*}
  \frac{\delta}{r}\norm{\widehat b_i-b_i}=O(1),
\end{equation*}
which keeps the signed queries within a constant-width portion of the
tube.  
\makeatletter
\providecommand*{\theHALG@line}{}
\renewcommand*{\theHALG@line}{\thealgorithm.\arabic{ALG@line}}
\makeatother

\begin{algorithm}[t]
\caption{\texttt{ProgressiveTube}}
\label{alg:constructed-family-regret}
\begin{algorithmic}[1]
\State \textbf{Input:} Horizon $n$, noisy oracle access to
$f_{\theta^\star}$, and the known parameters $(r,\eps)$
\State Set $L\gets\log(edn)$ and compute $M_n,\delta,E$ from
\cref{eq:tube-regret-probing-scale}
\If{$\eps n\le C_0L^2M_n$}
  \State Query $a_t\gets0$ for every $t\in[n]$
  \State \textbf{Stop}
\EndIf
\State Fix an orthonormal basis $v_1,\ldots,v_d$
\If{$E<1$}
  \For{$i=1,\ldots,d$}
    \State Set $\widehat b_i\gets\texttt{TubeRefine}(v_i,E,L)$
  \EndFor
\Else
  \State Set $\widehat b_i\gets0$ for every $i\in[d]$
\EndIf
\State For every $i\in[d]$, set
$s_i\gets c\delta/(1+\norm{\widehat b_i})$ and query each of
\begin{equation*}
  a_i^\pm
  =\left(\pm s_iv_i,
    \pm\frac{8\eps s_i}{r}\widehat b_i\right)
\end{equation*}
exactly $m=\lceil C_1L\sqrt n/(\eps\delta)\rceil$ times
\State Use the signed exponential inversion in
\cref{eq:regret-upper-coordinate-inversion} to form $\widehat\beta_i$ for
every $i\in[d]$, and set
\begin{equation*}
  \widehat u
  \gets
  \Pi_{(3/8)\bbB_2^d}
  \left(\sum_{i=1}^d\widehat\beta_iv_i\right)
\end{equation*}
\State Run $\texttt{PointTrack}(\widehat u,N,L)$, where $N$ is the number
of remaining rounds
\end{algorithmic}
\end{algorithm}

\medskip
\noindent The complete implementations of the two longer subroutines are
given in \Cref{alg:tube-refine,alg:point-track} in the appendix.  We next
explain what each stage does and how its regret enters the final bound.

\medskip
\noindent\textit{Stage 1: localizing the required tube directions.}
Fix $v_i$ and write $b_i=W^\star v_i$.  The purpose of
\texttt{TubeRefine} is not to estimate $b_i$ to the noise level; it only
reduces its error to $E=r/\delta$.  Suppose the current estimate
$\overline b$ satisfies $\norm{\overline b-b_i}\le R$.  The subroutine
uniformly explores $\{\overline b \pm ARe_i\}_{i\in [d]}$ in
the second block, for a large constant $A>0$, and a least-squares fit to these distances
contracts the localization radius from $R$ to $qR$ using $
  \widetilde O(\frac{d^2r^2}{\eps^2R^2}) $
queries, for some constant contraction factor $q\in (0,1)$. The per-round regret during this process is at most $\widetilde{O}(\frac{\varepsilon R}{r})$. Therefore, repeating this process until $R$ decreases geometrically to $E$, localizing
one $b_i$ to error $E$ costs
\begin{equation*}
  \widetilde O\left(
    \sum_{k\ge 0}
      \frac{d^2r^2}{\eps^2(q^{-k}E)^2}
      \frac{\eps q^{-k} E}{r}
  \right)
  =\widetilde O\left(\frac{d^2r}{\eps E}\right).
\end{equation*}
Repeating this for all $d$ basis directions gives
\begin{equation}
  R_{\mathrm{map}}
  =\widetilde O\left(\frac{d^3r}{\eps E}\right)
  =\widetilde O\left(\frac{d^3\delta}{\eps}\right).
  \label{eq:tube-refine-summary}
\end{equation}
This cost is present only when $\delta>r$; otherwise $E=1$ and the entire
stage is skipped.

\medskip
\noindent\textit{Stage 2: estimating the target by signed cancellation.}
Let $\beta_i=\langle u^\star,v_i\rangle$.  The signed pair $a_i^\pm$ in
\Cref{alg:constructed-family-regret} has the same tube argument
\begin{equation*}
  q_i=\frac{s_i}{r}\norm{b_i-\widehat b_i}=O(1)
\end{equation*}
for both signs.  In the absence of observation noise,
\begin{equation*}
  2\exp\left(\frac{f_{\theta^\star}(a_i^\pm)}{\eps}\right)
  =e^{q_i}+e^{s_i^2/2\mp s_i\beta_i}.
\end{equation*}
Subtracting the two signs removes $e^{q_i}$ exactly, so $\beta_i$ can be
recovered even though $b_i$ is known only approximately.  Averaging each
sign $m$ times and applying the Gaussian exponential bias correction gives
\begin{equation*}
  \bbE\norm{\widehat u-u^\star}^2
  \lesssim\frac{d}{m\eps^2\delta^2}
  \le\frac{Cd}{\eps\delta\sqrt n}.
\end{equation*}
There are $2dm$ signed queries and each costs $O(\eps)$ regret, so their
exploration regret is $\widetilde O(d\sqrt n/\delta)$.  The target error
during the remaining rounds contributes
\begin{equation*}
  n\eps\,\bbE\norm{\widehat u-u^\star}^2
  =\widetilde O\left(\frac{d\sqrt n}{\delta}\right)
\end{equation*}
as well.  This explains why a larger probing scale $\delta$ makes the
target-learning stage cheaper.

\medskip
\noindent\textit{Stage 3: final tube tracking.}
After fixing $\widehat u$, the learner only needs to track the single image
$W^\star\widehat u$.  The subroutine \texttt{PointTrack} in the appendix applies the same
localization idea as in Stage 1 to this one vector while keeping the first
block fixed at $\widehat u$; it introduces no additional statistical
ingredient.  A geometric sequence of localization epochs gives total
regret $\widetilde O(d\sqrt n)$.  Since $n\ge Cd^2$, this term is no larger
than either rate in \cref{eq:constructed-family-regret-upper-bound} and does
not affect the final bound. 

\medskip
\noindent\textit{Putting the stages together.}
Ignoring logarithmic factors, the zero-action branch costs $O(\eps n)$,
whereas the structured branch costs
\begin{equation}
  \mathfrak R_n(\theta^\star)
  \lesssim
  \Ind_{\{\delta>r\}}\frac{d^3\delta}{\eps}
    +\frac{d\sqrt n}{\delta}. 
  \label{eq:progressive-tube-summary}
\end{equation}
Plugging in the parameters and distinguishing into different cases yields the target regret bound.

%% file: appendix.tex
\section{Properties of the Convex Function Family}
\label{sec:construction-properties}

For $\theta=(W,u) \in \bbR^{d \times d} \times \bbR^d$ and $a=(\aone,\atwo) \in \bbR^d \times \bbR^d$, recall that
\begin{equation*}
\begin{aligned}
    f^1_\theta(a)
    &\triangleq\frac1r\norm{W\aone-\frac{r}{8\eps}\atwo}, \\
    f^2_\theta(a)
    &\triangleq\frac12\norm{\aone-u}^2-\frac12\norm{u}^2, \\
    f_\theta(a)
    &\triangleq\eps\smax\bigl(f^1_\theta(a),f^2_\theta(a)\bigr).
\end{aligned}
\end{equation*}
The proof of \Cref{thm:convex-optimization-lower-bound} uses two elementary properties of the ground-truth function $f_{\theta^\star}$. First, small error forces the first block of the action to be close to $u^\star$. Second, a draw from the Gaussian prior is $1$-Lipschitz with high probability.  We prove these facts here. For the action $a_t$ at round $t$, we use
$f_t^i=f^i_{\theta^\star}(a_t)$ for $i \in \{ 1,2\}$.

\begin{lemma}[Minimizer and Optimization Error]
\label{lem:target-excess-controls-estimation}
The loss $f_{\theta^\star}$ is convex on $\bbR^{2d}$ and satisfies
$f_{\theta^\star}(0)=0$.  Its global minimizer is
\begin{equation}
    a_{\theta^\star}
    =\left(u^\star,\frac{8\eps}{r}W^\star u^\star\right).
    \label{eq:hidden-map-unconstrained-benchmark}
\end{equation}
Moreover, there is a universal $c>0$ such that
\begin{equation}
    f_{\theta^\star}(a)-f_{\theta^\star}(a_{\theta^\star})
    \geq
    c\max\left\{\frac{\eps}{r}\|W^\star a^1 - \frac{r}{8\varepsilon}a^2\| , \eps \norm{\aone-u^\star}^2\Ind_{\|u^\star\|\le 2}\right\}.
    \label{eq:target-excess-mse}
\end{equation}
\end{lemma}

\begin{proof}
$f_\theta$ is the scaled softmax of two convex functions. The first is the norm of an affine map, and the second is a convex quadratic. Since $\smax$ is convex and nondecreasing in each argument, this implies that $f_{\theta^\star}$ is convex. Furthermore, both convex functions vanish at the origin, and by the definition of $\smax$ this implies $f_{\theta^\star}(0)=0$.

\medskip
\noindent The smallest possible value of $f^1_{\theta^\star}$ is zero, attained when $\atwo=(8\eps/r)W^\star\aone$. The second branch is minimized at
$\aone=u^\star$. These conditions hold simultaneously at
\cref{eq:hidden-map-unconstrained-benchmark}; monotonicity of $\smax$ shows
that this point is a global minimizer.

\medskip
\noindent It remains to relate optimization error to estimation error (i.e., prove \cref{eq:target-excess-mse}). Define
$h(x,y)=\varepsilon[\log(e^x+e^y)-\log 2]$. For the first lower bound, 
\begin{align*}
f_{\theta^\star}(a)-f_{\theta^\star}(a^\star_{\theta^\star}) \ge  h\left(\frac{1}{r}\|W^\star a^1 - \frac{r}{8\varepsilon}a^2\|, -\frac{\|u^\star\|^2}{2}\right) - h\left(0, -\frac{\|u^\star\|^2}{2}\right) \ge \frac{\eps}{2r}\|W^\star a^1 - \frac{r}{8\varepsilon}a^2\|, 
\end{align*}
where the last step uses $\partial_x h(x,y) = \frac{\eps e^x}{e^x+e^y}\ge \frac{\eps}{2}$ whenever $x\ge 0$ and $y\le 0$. For the second lower bound, for $\|u^\star\|\le 2$ we use
\begin{align*}
    f_{\theta^\star}(a)-f_{\theta^\star}(a_{\theta^\star})
    &\geq
      h(0,f^2_{\theta^\star}(a))
      -h(0,f^2_{\theta^\star}(a_{\theta^\star})) \\
    &\geq
      c\eps \bigl(
        f^2_{\theta^\star}(a)
        -f^2_{\theta^\star}(a_{\theta^\star})
      \bigr)
     =\frac{c\eps}{2}\norm{\aone-u^\star}^2.
\end{align*}
The second inequality follows because
$f^2_{\theta^\star}(a)\geq
f^2_{\theta^\star}(a_{\theta^\star})
\ge -\norm{u^\star}^2/2\ge -2$, and hence
$\partial_y h(0,y)=\eps e^y/(1+e^y)\geq c\eps $ throughout the relevant interval.
\end{proof}

\begin{lemma}[High-Probability Lipschitzness]
\label{lem:high-probability-lipschitz}
Under the Gaussian prior \cref{eq:gaussian-priors}, suppose that
$\eps/r\leq1/16$ and $\eps\leq1/2$.  Under the event
\begin{equation}
    \calE_{\mathrm{b}}
    \triangleq \left\{
      \norm{W^\star}_{\mathrm{op}}\le4,
      \ \norm{u^\star}\le\frac38
    \right\},
    \label{eq:bounded-prior-event}
\end{equation}
the function $f_{\theta^\star}$ is
$1$-Lipschitz on $\bbB^{2d}$. Moreover,
$\Pr(\calE_{\mathrm{b}})\geq1-Ce^{-cd}$.
\end{lemma}
\begin{proof}
Under the event $\calE_{\mathrm{b}}$, the two branches within the softmax of $f_{\theta^\star}$ have Lipschitz constants
\begin{equation*}
\begin{aligned}
    \operatorname{Lip}(\eps f^1_{\theta^\star})
    &\leq
      \sqrt{
        \frac1{64}+\frac{\eps^2}{r^2}\norm{W^\star}_{\mathrm{op}}^2
      }
      \leq\frac{1}2, \\
    \operatorname{Lip}(\eps f^2_{\theta^\star})
    &\leq
      \eps(1+\norm{u^\star})
      \leq\frac{11}{16}.
\end{aligned}
\end{equation*}
Every subgradient of the soft maximum is a convex combination of
subgradients of these two scaled branches.  Hence $f_{\theta^\star}$ is
$1$-Lipschitz on $\calE_{\mathrm{b}}$. 

\medskip
\noindent To prove the probability bound, write $W^\star=d^{-1/2}G$ and $u^\star=(4\sqrt d)^{-1}g$, where $G$ and
$g$ have independent standard Gaussian entries. Gaussian vector and
random-matrix concentration
\citep[Theorem~3.1.1 and Corollary~7.3.3]{vershynin2018} give
\begin{equation*}
    \Pr(\norm{W^\star}_{\mathrm{op}}>4)
    +\Pr(\norm{u^\star}>3/8)
    \le Ce^{-cd}.
\end{equation*}
The claim follows from the definition of $\calE_{\mathrm{b}}$ and a union
bound.
\end{proof}

\input{appendix-algorithm}
\input{appendix-unconstrained}

%% file: appendix-algorithm.tex
\section{Proof of
\Cref{thm:constructed-family-regret-upper-bound}}
\label{sec:constructed-family-regret-upper-bound-proof}

\noindent Fix an arbitrary
$(W^\star,u^\star)\in\operatorname{supp}(\pi_{\mathrm b})$.  Thus
\begin{equation}
  \norm{W^\star}_{\mathrm{op}}\le4,
  \qquad
  \norm{u^\star}\le\frac38.
  \label{eq:regret-upper-bounded-parameter}
\end{equation}
All expectations and probability statements below concern only the
observation noise and the learner's randomization.  Put $L=\log(edn)$.

\subsection{Stage 1: localizing the hidden map}

\noindent Let $v_1,\ldots,v_d$ be the basis in
\Cref{alg:constructed-family-regret}, put $b_i=W^\star v_i$, and recall
\begin{equation*}
  \delta
  =\min\left\{1,
    \max\left\{r,\frac{\sqrt\eps\,n^{1/4}}d\right\}
  \right\},
  \qquad
  E=1\wedge\frac r\delta.
\end{equation*}
Stage 1 is used only when $E<1$, which also implies $E\ge r$. Its statistical primitive is a nonlinear
regression for the center of a soft distance.  We give the regression
argument first and then instantiate it for each $b_i$. Let $e_1,\ldots,e_d$ denote the standard basis of $\bbR^d$.
The regression uses the paired coordinate probes
$\overline b\pm ARe_j$, with $\norm{b-\overline{b}} \le R$ and a large constant $A>0$. 

\begin{lemma}[Localizing the center of a soft distance]
\label{lem:soft-distance-localization}
Let $b\in\bbR^d$, $c_0>0$, $|h|\le \frac{1}{4}$, and suppose
the learner observes standard normal perturbations of
\begin{equation}
  \ell_{b,h}(z)
  =\eps\smax\left(\frac{c_0}{r}\norm{b-z},h\right).
  \label{eq:soft-distance-model}
\end{equation}
Assume that $\norm{b-\overline b}\le R$ and that the points
$\overline b+A\tau Re_j$, where $j\in[d]$ and
$\tau\in\{-1,1\}$, are feasible for the queries. For a fixed $c_0>0$ and large enough $A>0$, there is an
experiment using at most
\begin{equation}
  N_R\le\frac{Cd^2Lr^2}{\eps^2R^2}
  \label{eq:soft-distance-query-count}
\end{equation}
queries which returns $\overline b^+$ satisfying for some universal constant $q\in (0,1)$, 
\begin{equation}
  \Pr\left(\norm{\overline b^+-b}>qR\right)\le n^{-5}.
  \label{eq:soft-distance-contraction}
\end{equation}
Every queried $z$ satisfies $\norm{z-b}\le CR$, and one can take $q=\frac{17}{18}$. 
\end{lemma}

\begin{proof}
For $z_j^\tau=\overline b+A\tau Re_j$, repeat every one of the $2d$
design points
\begin{equation*}
  m_R=\left\lceil\frac{CdLr^2}{\eps^2R^2}\right\rceil
\end{equation*}
times.  Write $\overline Y_j^\tau$ for the averages and choose
$(\overline b^+,\widehat h)$ by least squares over
$\{(b',h'): \norm{b'-\overline b}\le R, |h'|\le 1/4\}$. 

\medskip
\noindent We start from the basic inequality for least squares.  Let
$\Phi(b',h')=(\ell_{b',h'}(z_j^\tau))_{j,\tau}$ and let
$\overline\xi\in\bbR^{2d}$ be the averaged noise vector.  Its coordinates
are independent $\calN(0,1/m_R)$ variables, so
\begin{equation*}
  \norm{\overline\xi}
  \le C\sqrt{\frac{dL}{m_R}} \le \frac{c_0\eps R}{48r}
\end{equation*}
with probability at least $1-n^{-5}$, by adjusting constants. The basic inequality for least
squares gives
\begin{align*}
  \norm{\Phi(\overline b^+,\widehat h)-\Phi(b,h)}^2
  \le2\left\langle
    \overline\xi,
    \Phi(\overline b^+,\widehat h)-\Phi(b,h)
  \right\rangle \le2\norm{\overline\xi}\cdot 
    \norm{\Phi(\overline b^+,\widehat h)-\Phi(b,h)}.
\end{align*}
Consequently, on the same event,
\begin{equation}
  \norm{\Phi(\overline b^+,\widehat h)-\Phi(b,h)}
  \le \frac{c_0\eps R}{24r}. 
  \label{eq:soft-distance-prediction-error}
\end{equation}
Next we show that \cref{eq:soft-distance-prediction-error} implies $\norm{\overline b^+ - b}\le qR$. Define
\begin{align*}
\Delta(b,h) \triangleq (\Delta_j(b,h))_{j\in [d]}, \qquad \text{ with } \Delta_j(b,h) \triangleq \ell_{b,h}(z_j^-) - \ell_{b,h}(z_j^+). 
\end{align*}
It is clear that
\begin{align}\label{eq:delta-difference}
\norm{\Delta(\overline b^+,\widehat h)-\Delta(b,h)} \le \sqrt{2} \norm{\Phi(\overline b^+,\widehat h)-\Phi(b,h)} \le \frac{c_0\eps R}{12r}. 
\end{align}
By the mean value theorem, there exists some $\xi$ lying between $\norm{z_j^+ - b}$ and $\norm{z_j^- -b}$ such that
\begin{align*}
\Delta_j(b,h) = \frac{c_0\eps}{r} \frac{e^{c_0\xi/r}}{e^{c_0\xi/r} + e^h} (\norm{z_j^- -b} - \norm{z_j^+ - b}). 
\end{align*}
Since $|h|\le \frac{1}{4}$ and $c_0\xi/r\ge 0$, the ratio lies in $[\frac{5}{12}+0.01, 1]$. In addition,
\begin{align*}
\norm{z_j^- -b} - \norm{z_j^+ - b} = \frac{4AR(\overline{b}_j - b_j)}{\norm{z_j^- -b} + \norm{z_j^+ - b}} \in \Big[ \frac{4A(\overline{b}_j - b_j)}{2(A\pm 1)} \Big]. 
\end{align*}
Now taking $A$ large enough, we have
\begin{align}\label{eq:approx-error}
| \Delta_j(b,h) - \frac{3c_0\eps}{2r}(\overline{b}_j - b_j) | \le \frac{2c_0\eps}{3r}|\overline{b}_j - b_j|. 
\end{align}
Plugging \cref{eq:approx-error} with $b$ and $\overline b^+$ into \cref{eq:delta-difference} gives
\begin{align*}
\frac{c_0\eps R}{12r} \ge \frac{3c_0\eps}{2r}\norm{ \overline b^+ - b } - \frac{2c_0\eps}{3r}(\norm{\overline b - b} + \norm{\overline{b} - \overline{b}^+})  \ge \frac{3c_0\eps}{2r}\norm{ \overline b^+ - b } - \frac{4c_0\eps R}{3r}. 
\end{align*}
Solving the inequality yields $\norm{\overline b^+ - b}\le \frac{17}{18}R$, so that $q=\frac{17}{18}<1$. 
\end{proof}

\medskip
\noindent Iterating the lemma over radii
$R_j=q^j R_0$ gives the following consequence.  We stop as soon as
$R_j\le CE$, and project each estimate onto the known parameter ball.
Because both the sample sizes and the incurred regret are geometric sums,
\begin{align}
  N(E)&\le\frac{Cd^2Lr^2}{\eps^2E^2},
  \label{eq:tube-refine-total-queries}\\
  \sum_jN_{R_j}\frac{C\eps R_j}{r}
  &\le\frac{Cd^2Lr}{\eps E}.
  \label{eq:tube-refine-one-direction-regret}
\end{align}
The probability that any contraction fails is at most $n^{-4}$.  The
complete map-localization subroutine is as follows. When $E=1$, the
deterministic estimate
$\widehat b_i=0$ already satisfies
$\norm{\widehat b_i-b_i}\le4E$.  When $E<1$, we run \texttt{TubeRefine} in
\Cref{alg:tube-refine}. Note that for a small enough $c_0$, all actions will lie in $\bbB^{2d}$, and the second softmax component has magnitude at most $\frac{1}{4}$. 

\begin{algorithm}[!t]
\caption{\texttt{TubeRefine}}
\label{alg:tube-refine}
\begin{algorithmic}[1]
\State \textbf{Input:} A unit vector $v$, a target radius $E<1$, and
$L=\log(edn)$
\State Set $\overline b\gets0$ and $R\gets4$
\While{$R>E$}
    \State Let constants $q, A, C>0$ be given in \Cref{lem:soft-distance-localization}, and $c_0>0$ be small enough
  \State Set
  $m_R\gets\lceil CdLr^2/(\eps^2R^2)\rceil$
  \State For each $(\tau,j)\in\{-1,1\}\times[d]$, query
  \begin{equation*}
 \left(
      c_0v,\frac{8c_0\eps}{r}
      (\overline b+A\tau Re_j)
    \right)
  \end{equation*}
  exactly $m_R$ times, and denote the average observation by
  $\overline Y_j^\tau$
  \State Choose $(\widetilde b,\widetilde h)$ to minimize
  \begin{equation*}
    \sum_{\tau,j}\left[
      \overline Y_j^\tau
      -\eps\smax\left(
        \frac{c_0}{r}
        \norm{b-\overline b-A\tau Re_j},h
      \right)
    \right]^2
  \end{equation*}
  over $\norm{b-\overline b}\le R$ and $|h|\le \frac{1}{4}$
  \State Set
  $\overline b\gets\Pi_{4\bbB_2^d}(\widetilde b)$ and
  $R\gets qR$
\EndWhile
\State \textbf{Return:} $\overline b$
\end{algorithmic}
\end{algorithm}

\medskip
\noindent At localization radius $R\ge E$, the two arguments of the
softmax and the benchmark value show that every query has regret at most
$C\eps R/r$.  Consequently,
\cref{eq:tube-refine-one-direction-regret} and a union bound over the basis
give estimates satisfying
\begin{equation}
  \max_{i\in[d]}\norm{\widehat b_i-b_i}\le CE
  \label{eq:map-refine-accuracy}
\end{equation}
with probability at least $1-n^{-3}$, at total regret
\begin{equation}
  R_{\mathrm{map}}
  \le\frac{Cd^3Lr}{\eps E}
  =\frac{Cd^3L\delta}{\eps}.
  \label{eq:map-refine-regret}
\end{equation}
The number of queries used is at most
\begin{equation}
  N_{\mathrm{map}}
  \le\frac{Cd^3Lr^2}{\eps^2E^2}
  =\frac{Cd^3L\delta^2}{\eps^2}.
  \label{eq:map-refine-query-count}
\end{equation}

\subsection{Stage 2: estimating the target by signed cancellation}

\noindent We first derive the signed inversion used in this stage.  Fix a
unit vector $v$, a scalar $s>0$, and $b_0\in\bbR^d$, and consider
\begin{equation*}
  a^+=\left(sv,\frac{8\eps s}{r}b_0\right),
  \qquad
  a^-=\left(-sv,-\frac{8\eps s}{r}b_0\right).
\end{equation*}
Writing $b=W^\star v$ and $\beta=\langle u^\star,v\rangle$, the two
softmax arguments are
\begin{equation*}
  q=\frac{s}{r}\norm{b-b_0},
  \qquad
  h^\pm=\frac{s^2}{2}\mp s\beta.
\end{equation*}
In particular, the tube argument is identical for the two signs.  If
$\mu^\pm=f_{\theta^\star}(a^\pm)$ and
$B^\pm=2e^{\mu^\pm/\eps}$, then
\begin{equation}
  B^\pm=e^q+e^{s^2/2\mp s\beta},
  \qquad
  B^--B^+=2e^{s^2/2}\sinh(s\beta),
  \label{eq:regret-upper-symmetric-identity}
\end{equation}
and hence
\begin{equation}
  \beta
  =\frac1s\operatorname{arcsinh}\left(
    \frac{e^{-s^2/2}}2(B^--B^+)
  \right).
  \label{eq:regret-upper-coordinate-inversion}
\end{equation}
Thus the remaining error in $b_0$ affects the size of the observations but
cancels from their signed difference.

\medskip
\noindent Suppose that each sign is queried $m$ times, and let
$\overline Y^\pm$ be the two observation averages.  Define
\begin{equation}
  \widehat B^\pm
  \triangleq
  2\exp\left(
    \frac{\overline Y^\pm}{\eps}-\frac1{2m\eps^2}
  \right).
  \label{eq:regret-upper-lognormal-correction}
\end{equation}
The Gaussian moment-generating function gives
\begin{equation}
  \bbE[\widehat B^\pm]=B^\pm,
  \qquad
  \operatorname{Var}(\widehat B^\pm)
  =(B^\pm)^2\left(e^{1/(m\eps^2)}-1\right).
  \label{eq:regret-upper-lognormal-moments}
\end{equation}
If $q\le C_0$, $s\le c_0$, and $m\eps^2\ge1$, all the $B^\pm$ are
bounded by a universal constant.  Since $\operatorname{arcsinh}$ is
$1$-Lipschitz, inserting $\widehat B^\pm$ into
\cref{eq:regret-upper-coordinate-inversion}, and then clipping the result
to $[-3/8,3/8]$, yields
\begin{equation}
  \bbE[(\widehat\beta-\beta)^2]
  \le\frac{C}{m\eps^2s^2}.
  \label{eq:regret-upper-coordinate-error}
\end{equation}

\medskip
\noindent We now apply this identity coordinate by coordinate.  On the
event in \cref{eq:map-refine-accuracy}, set
\begin{equation*}
  s_i=\frac{c_1\delta}{1+\norm{\widehat b_i}},
  \qquad
  a_i^\pm
  =\left(\pm s_iv_i,
    \pm\frac{8\eps s_i}{r}\widehat b_i\right).
\end{equation*}
If $E=1$, use the same definition with $\widehat b_i=0$.  Since
$\norm{b_i}\le4$, clipping $\widehat b_i$ to $4\bbB_2^d$ ensures
$c\delta\le s_i\le C\delta$.  Moreover,
\begin{equation*}
  \frac{s_i}{r}\norm{b_i-\widehat b_i}
  \le C\frac{\delta E}{r}\le C.
\end{equation*}
Thus the hypotheses of \cref{eq:regret-upper-coordinate-error} hold, and
the actions are feasible because $r\ge 16\eps$ and $c_1$ is sufficiently
small.

\medskip
\noindent Query each sign
\begin{equation*}
  m=\left\lceil\frac{C L\sqrt n}{\eps\delta}\right\rceil
\end{equation*}
times, form $\widehat\beta_i$ by
\cref{eq:regret-upper-coordinate-inversion,eq:regret-upper-lognormal-correction},
and project $\sum_i\widehat\beta_iv_i$ onto $(3/8)\bbB_2^d$.  Euclidean
projection cannot increase its distance from $u^\star$.  Therefore
\cref{eq:regret-upper-coordinate-error} and orthogonality give
\begin{equation}
  \bbE\norm{\widehat u-u^\star}^2
  \le\frac{Cd}{m\eps^2\delta^2}+Cn^{-3}
  \le\frac{CdL}{\eps\delta\sqrt n}.
  \label{eq:target-estimation-error-regret}
\end{equation}
Every signed query has both softmax arguments bounded by a universal
constant, and hence has regret at most $C\eps$.  The target-exploration
regret is consequently
\begin{equation}
  R_{\mathrm{target}}
  \le C\eps dm
  \le\frac{CdL\sqrt n}{\delta}.
  \label{eq:target-exploration-regret}
\end{equation}

\subsection{Stage 3: tracking the tube at the estimated target}

\noindent This last stage contains no new statistical idea.  With the
first block fixed at $x=\widehat u$, it applies the same paired-coordinate
regression as Stage 1 to the single vector $b=W^\star x$, shrinking the
localization radius whenever the next epoch fits in the remaining budget.
We include the short implementation for completeness.  The constants
$c_2>0$ and $C_2>0$ are sufficiently small and large, respectively.

\begin{algorithm}[!t]
\caption{\texttt{PointTrack}}
\label{alg:point-track}
\begin{algorithmic}[1]
\State \textbf{Input:} A fixed first block
$x\in(3/8)\bbB_2^d$, a budget $N$, and $L=\log(edn)$
\State Set $q_0\gets7/4$,
$m_0\gets\lceil C_2dLr^2/\eps^2\rceil$, and $N_0\gets2dm_0$
\If{$N<N_0$}
  \State Query $(x,0)$ for all $N$ rounds and \textbf{stop}
\EndIf
\State At each point $\tau q_0e_j$, take $m_0$ observations using the
action $(x,(8\eps/r)\tau q_0e_j)$, and fit $(\overline b,\overline h)$ by
nonlinear least squares under the model in
\cref{eq:soft-distance-model}, over $\norm b\le3/2$ and $|h|\le1/4$
\State Set $R\gets c_2$ and $N_{\mathrm{used}}\gets N_0$
\While{the next paired-coordinate epoch fits in the budget}
  \State Set
  $m_R\gets\lceil C_2dLr^2/(\eps^2R^2)\rceil$ and
  $N_R\gets2dm_R$
  \State Apply the paired-coordinate regression in
  \Cref{lem:soft-distance-localization} at center $\overline b$ and radius
  $R$, using actions
  $(x,(8\eps/r)(\overline b+A\tau Re_j))$
  \State Set
  $\overline b$ to the projected regression output,
  $R\gets qR$ ($q$ in \Cref{lem:soft-distance-localization}), and
  $N_{\mathrm{used}}\gets N_{\mathrm{used}}+N_R$
\EndWhile
\State For every unused round, query
$\left(x,(8\eps/r)\overline b\right)$
\end{algorithmic}
\end{algorithm}

\begin{lemma}[Tracking one tube center]
\label{lem:one-tube-center-tracking}
Fix $x\in(3/8)\bbB_2^d$ and let $b=W^\star x$.  There is a policy
which keeps its first block equal to $x$ and whose tube regret over $N$
rounds satisfies
\begin{equation}
  \bbE\left[
    \sum_{t=1}^N
    \left(
      f_{\theta^\star}\left(x,\frac{8\eps}{r}z_t\right)
      -f_{\theta^\star}\left(x,\frac{8\eps}{r}b\right)
    \right)
  \right]
  \le Cd\sqrt{NL}.
  \label{eq:one-tube-center-regret}
\end{equation}
\end{lemma}

\begin{proof}
For fixed $x$, the observations have the form in
\cref{eq:soft-distance-model}, with
\begin{equation*}
  h=\frac12\norm{x}^2-\langle x,u^\star\rangle. 
\end{equation*}
Since $\norm{x}, \norm{u^\star}\le \frac{3}{8}$, we have $|h|\le \frac{1}{4}$, satisfying the condition of \Cref{lem:soft-distance-localization}. Also $\norm b\le \frac{3}{2}$.  We analyze the
explicit policy in \Cref{alg:point-track}.  If
$N<Cd^2Lr^2/\eps^2$, use $z_t=0$.  The softmax is $1$-Lipschitz in each argument, so the left-hand side of
\cref{eq:one-tube-center-regret} is at most
\begin{equation*}
  \frac{C\eps N}{r}\le Cd\sqrt{NL}.
\end{equation*}

\medskip
\noindent Otherwise, the initial paired-coordinate regression at radius
$7/4$ obtains an estimate of $b$ with a sufficiently small constant error
$c_2$.  The calculation in
\Cref{lem:soft-distance-localization}, now on the compact set
$\{\norm b\le3/2,\ |h|\le1/4\}$, proves its asserted accuracy when $C_2$ is
large enough.  This design is feasible because its physical
second block has norm at most $7/8$, whereas $\norm{x}\le3/8$.  It costs
$Cd^2Lr^2/\eps^2$ queries and $Cd^2Lr/\eps$ regret, both bounded by
$Cd\sqrt{NL}$ in the present case.  Project the estimate onto
$(3/2)\bbB_2^d$.  From this point on, sufficiently small
paired-coordinate neighborhoods of the current center are feasible.
Apply \Cref{lem:soft-distance-localization} at geometrically decreasing radii.
At radius $R$, use
$N_R=Cd^2Lr^2/(\eps^2R^2)$ queries.  The instantaneous tube regret is at
most $C\eps R/r$, so this epoch contributes at most
$Cd^2Lr/(\eps R)=Cd\sqrt{LN_R}$.  Stop before the first epoch that would
exceed the remaining horizon and use the last center thereafter.  Since
the epoch lengths grow geometrically,
$\sum_j\sqrt{N_{R_j}}\le C\sqrt N$, proving
\cref{eq:one-tube-center-regret}.  The feasibility follows from
$\norm x\le3/8$, $\norm b\le3/2$, and $r/(8\eps)\ge2$; the local radii
are chosen as sufficiently small universal constants.  The union of the
localization failures has probability at most $n^{-3}$; clipping all
centers to $(3/2)\bbB_2^d$ makes its contribution to expected regret
negligible.
\end{proof}

\medskip
\noindent Apply \Cref{lem:one-tube-center-tracking} with $x=\widehat u$.
For every $z$, the $1$-Lipschitzness of the softmax in each argument gives
\begin{align}
  f_{\theta^\star}\left(
    \widehat u,\frac{8\eps}{r}z
  \right)-f_{\theta^\star}(a_{\theta^\star}) \le
  \frac{\eps}{r}\norm{z-W^\star\widehat u}
  +\frac\eps2\norm{\widehat u-u^\star}^2.
  \label{eq:target-tube-regret-decomposition}
\end{align}
Combining
\cref{eq:target-estimation-error-regret,eq:one-tube-center-regret,eq:target-tube-regret-decomposition}
shows that all rounds after target exploration contribute at most
\begin{equation}
  \frac{CdL\sqrt n}{\delta}+Cd\sqrt{nL}
  \label{eq:post-target-regret}
\end{equation}
expected regret.

\subsection{Regret bound and parameter regimes}

\noindent The loss range on $\bbB_2^{2d}$ is bounded by a universal
constant under \cref{eq:regret-upper-bounded-parameter} and
$r\ge 16\eps$.  Hence all localization failures together contribute at
most $C/n$ to expected regret.  Equations
\cref{eq:map-refine-regret,eq:target-exploration-regret,eq:post-target-regret}
give
\begin{equation}
  \mathfrak R_n(\theta^\star)
  \le CL^2\left(
    \Ind_{\{\delta>r\}}\frac{d^3\delta}{\eps}
    +\frac{d\sqrt n}{\delta}
  \right)
  \label{eq:constructed-family-parameter-dependent-regret}
\end{equation}
whenever the structured branch is used. We show that this implies the target regret bound. 

\medskip
\noindent Recall
$M_n=\min\{\sqrt d\,n^{3/4},d^{4/3}\sqrt n\}$.  If the zero-action branch
is used, then $f_{\theta^\star}(0)=0$ and
$f_{\theta^\star}(a_{\theta^\star})\ge-C\eps$, so its regret is at most
$C\eps n\le CL^2M_n$. In the sequel we assume that the zero-action test fails. 

\medskip
\noindent Suppose first that $d^2\le n\le d^{10/3}$.  Then
$M_n=\sqrt d\,n^{3/4}$.  Since the zero-action test fails, 
\begin{equation*}
  \eps\gtrsim\frac{M_n}{n}
  =\sqrt d\,n^{-1/4}.
\end{equation*}
The relation $r\ge 16\eps$ and $n\le d^{10/3}$ imply $\delta=r\wedge1$, and therefore the first term of \cref{eq:constructed-family-parameter-dependent-regret} is zero. The regret upper bound is then
\begin{equation*}
  \frac{d\sqrt n}{\delta}
  =\frac{d\sqrt n}{r\wedge 1}
  \lesssim \frac{d\sqrt n}{\eps}
  \lesssim\sqrt d\,n^{3/4}=M_n.
\end{equation*}

\medskip
\noindent Now suppose that $n>d^{10/3}$, so
$M_n=d^{4/3}\sqrt n$, and put
\begin{equation*}
  \delta_\star=\frac{\sqrt\eps\,n^{1/4}}d.
\end{equation*}
The failure of the zero-action test gives
$\eps\gtrsim d^{4/3}/\sqrt n$.  If
$r\ge\delta_\star$, then $E=1$, Stage 1 is skipped, and
\begin{equation*}
  \frac{d\sqrt n}{\delta}
  \le\frac{d\sqrt n}{\delta_\star \wedge 1}
  \le \frac{d^2n^{1/4}}{\sqrt\eps} + d\sqrt{n}
  \lesssim d^{4/3}\sqrt n=M_n.
\end{equation*}
\noindent If $r<\delta_\star\le1$, then $\delta=\delta_\star$ and the two
leading terms in \cref{eq:constructed-family-parameter-dependent-regret}
are equal:
\begin{equation*}
  \frac{d^3\delta}{\eps}
  =\frac{d\sqrt n}{\delta}
  =\frac{d^2n^{1/4}}{\sqrt\eps}
  \lesssim M_n.
\end{equation*}
Finally, if $\delta_\star>1$, then $\delta=1$ and
$\eps>d^2/\sqrt n$.  Hence
\begin{equation*}
  \frac{d^3}{\eps}<d\sqrt n\le M_n.
\end{equation*}
This proves \cref{eq:constructed-family-regret-upper-bound} in every
parameter regime.

%% file: appendix-unconstrained.tex
\section{Extension to Unconstrained Action Spaces: Proof of \Cref{thm:unconstrained-convex-bandit-regret-lower-bound}}
\label{sec:conic-extension-proofs}

In this appendix we use $d$ for the block dimension, so the hard functions
initially live on $\bbR^{2d}$. We use the Gaussian prior
in \cref{eq:gaussian-priors}, write $\pi$ for this prior, and write
$\pi_{\mathrm b}=\pi(\,\cdot\mid\calE_{\mathrm b})$ for its restriction to
the event in \cref{eq:bounded-prior-event}.  The parameters used below always
satisfy $\eps\le1/2$ and
\begin{equation}
    \frac{\eps}{r}\le \frac1{16}.
    \label{eq:unconstrained-parameter-condition}
\end{equation}
In particular, on $\calE_{\mathrm b}$ the minimizer in
\cref{eq:hidden-map-unconstrained-benchmark} obeys
\begin{equation*}
  \norm{a_\theta}
  \le \left[\left(\frac38\right)^2+
      \left(\frac{8\eps}{r}\,4\,\frac38\right)^2\right]^{1/2}
  \le \frac{3\sqrt5}{8}.
\end{equation*}

\medskip
\noindent For the construction in \Cref{sec:unconstrained-action-spaces}, take
\begin{equation*}
    \lambda=\frac1{64},
    \qquad
    \varrho=\frac{3\sqrt5}{8}<\frac{29}{33},
    \qquad
    \mathfrak r_0=\frac{30}{31}.
\end{equation*}
Define the smoothed hinge by
\begin{equation}
    \Psi(s)
    \triangleq
    \begin{cases}
      0,&s\le-\lambda,\\
      \lambda\left(q^3-\frac12q^4\right),
        \quad q=(s+\lambda)/\lambda,&-\lambda<s<0,\\
      s+\lambda/2,&s\ge0.
    \end{cases}
    \label{eq:smoothed-hinge-definition}
\end{equation}
A direct calculation gives
\begin{equation}
    0\le\Psi'\le1,
    \qquad
    0\le\Psi''\le\frac{3}{2\lambda},
    \qquad
    \Psi(s)\ge s+\lambda/2.
    \label{eq:smoothed-hinge-properties}
\end{equation}

\subsection{Proof of \Cref{lem:conic-extension}}

We first verify the deterministic properties of the extension.  The map
$(a,z)\mapsto f_\theta(z)+\norm{a-z}$ is jointly convex, so partial
minimization over $z\in\bbB^{2d}$ shows that $\widetilde f_\theta$ is
convex.  The triangle inequality shows that it is globally $1$-Lipschitz.
The $1$-Lipschitzness of
$f_\theta$ on the ball gives
\begin{equation}
    \widetilde f_\theta=f_\theta
    \quad\text{on }\bbB^{2d}.
    \label{eq:extension-agrees-on-ball}
\end{equation}
Optimality of $a_\theta$ gives
$\widetilde f_\theta(a)\ge f_\theta(a_\theta)$, whereas choosing
$z=0$ and using $f_\theta(0)=0$ gives
$\widetilde f_\theta(a)\le\norm a$.

\medskip
\noindent Let $M(x,y)=y+\Psi(x-y-2\lambda)$.  The map $M$ is jointly convex, and
nondecreasing in both coordinates. Indeed, its
two partial derivatives are $\Psi'(x-y-2\lambda)$ and
$1-\Psi'(x-y-2\lambda)$; they are nonnegative and sum to one.  Since
$\lambda\widetilde f_\theta$ and $\Gamma$ are both $1/2$-Lipschitz,
$g_\theta=M(\lambda\widetilde f_\theta,\Gamma)$ is convex and globally
$1/2$-Lipschitz.

\medskip
\noindent Set
\begin{equation*}
    s_\theta(a)
    =\lambda\widetilde f_\theta(a)-\Gamma(a)-2\lambda.
\end{equation*}
The upper bound $\widetilde f_\theta(a)\le\norm a$ yields $s_\theta(a)\le-\lambda$ whenever
$\norm a\ge\mathfrak r_0$, and hence
\begin{equation}
    g_\theta(a)=\Gamma(a)
    \qquad\text{for all }\norm a\ge\mathfrak r_0.
    \label{eq:conic-parameter-free-collar}
\end{equation}
At the minimizer, $1$-Lipschitzness and $f_\theta(0)=0$ imply
$f_\theta(a_\theta)\ge-\norm{a_\theta}$.  Therefore
\begin{equation*}
  s_\theta(a_\theta)
  \ge \frac12-3\lambda-\left(\frac12+\lambda\right)
      \norm{a_\theta}
  \ge\frac{29-33\varrho}{64}>0,
\end{equation*}
and the last branch of \cref{eq:smoothed-hinge-definition} gives
\begin{equation}
    g_\theta(a_\theta)
    =\lambda f_\theta(a_\theta)-\frac{3\lambda}{2}.
    \label{eq:gthetastar}
\end{equation}
On the other hand, \cref{eq:smoothed-hinge-properties} and optimality of
$a_\theta$ give, for every $a$,
\begin{equation*}
  g_\theta(a)
  \ge\lambda\widetilde f_\theta(a)-\frac{3\lambda}{2}
  \ge\lambda f_\theta(a_\theta)-\frac{3\lambda}{2}.
\end{equation*}
Thus $a_\theta$ is a global minimizer of $g_\theta$.

\medskip
\noindent It remains to compare the two optimization errors.  If $\norm a\le1$,
then $\Pi(a)=a$ and \cref{eq:extension-agrees-on-ball} together with
\cref{eq:smoothed-hinge-properties,eq:gthetastar} gives
\begin{equation*}
  g_\theta(a)-g_\theta(a_\theta)
  \ge\lambda\bigl(f_\theta(a)-f_\theta(a_\theta)\bigr).
\end{equation*}
If $\norm a>1$, then $g_\theta(a)=\Gamma(a)\ge\lambda$.  Moreover,
$\Pi(a)\in\bbB_2^{2d}$ and hence
$f_\theta(\Pi(a))\le\norm{\Pi(a)}\le1$.  It follows that
\begin{align*}
  g_\theta(a)-g_\theta(a_\theta)
  =\Gamma(a)+\frac{3\lambda}{2}
      -\lambda f_\theta(a_\theta) \ge\lambda\bigl(f_\theta(\Pi(a))
      -f_\theta(a_\theta)\bigr).
\end{align*}
Combining the two cases with \cref{eq:target-excess-mse} proves, in fact,
the slightly stronger comparison
\begin{equation}
\begin{aligned}
  g_\theta(a)-g_\theta(a_\theta)
  &\ge\lambda\bigl(f_\theta(\Pi(a))-f_\theta(a_\theta)\bigr)\\
  &\ge c\lambda\max\left\{
      \frac{\eps}{r}
      \left\|W\Pi^1(a)-\frac{r}{8\eps}\Pi^2(a)\right\|,
      \eps\norm{\Pi^1(a)-u}^2
    \right\},
\end{aligned}
    \label{eq:conic-two-loss-domination}
\end{equation}
where the indicator in \cref{eq:target-excess-mse} is one under
$\pi_{\mathrm b}$.  This includes \cref{eq:conic-regret-domination}.
Since $g_\theta$ is convex, globally $1/2$-Lipschitz, and has a global
minimizer, it belongs to $\calF_{2d}(\bbR^{2d})$.

\subsection{Information bounds for the auxiliary mean}

We next transport the projection-based arguments of Sections~2 and~3 to the
auxiliary observation model.  In this subsection
$\theta^\star=(W^\star,u^\star)\sim\pi$, the learner may choose any
$a_t\in\bbR^{2d}$, and
\begin{equation*}
    r_t=\mu_{\theta^\star}(a_t)+Z_t,
    \qquad Z_t\stackrel{\mathrm{iid}}\sim N(0,1),
\end{equation*}
where $\mu$ is defined in
\cref{eq:smoothed-hinge-gaussian-auxiliary-mean}.  Expectations in this
model are denoted by $\bbE_{\pi,\mu}$.

\medskip
\noindent For $\norm{a_t}<1$, retain the notation $p_t$ from
\cref{eq:softmax-second-branch-weight}, and set the counterpart of \cref{eq:target-likelihood-fisher}:
\begin{equation}
\begin{aligned}
  \calI_{n,\mu}^u
  \triangleq \lambda^2\eps^2\sum_{t=1}^n
    \Ind_{\{\norm{a_t}<1\}}p_t^2\aone_t{\aone_t}^\top, \qquad 
  \overline{\calI}_{n,\mu}^u
  \triangleq\bbE_{\pi,\mu}
    [\calI_{n,\mu}^u\mid\calH_n].
\end{aligned}
    \label{eq:auxiliary-target-information}
\end{equation}
\noindent For later use, define the tube mismatch and the clipped, in-ball cumulative cost: 
\begin{equation}
        D_t\triangleq
    \left\|W^\star\aone_t-\frac{r}{8\eps}\atwo_t\right\|, \qquad \mathfrak D_n
    \triangleq\sum_{t=1}^n\bbE_{\pi,\mu}
      \left[\Ind_{\{\norm{a_t}<1\}}(D_t\wedge1)\right].
    \label{eq:auxiliary-clipped-tube-cost}
\end{equation}
The clipping will make the conditioning argument in the regret proof uniform
over all horizons. The following result is an extension of \Cref{lem:posterior-target-spread,lem:tension-estimation,lem:tension-regret}.

\begin{lemma}[Auxiliary posterior spread and projection bounds]
\label{lem:auxiliary-mean-information}
Suppose $d\ge3$ and \cref{eq:unconstrained-parameter-condition} holds.
There are universal constants $c_1,c,C>0$ with the following properties.
\begin{enumerate}
  \item If an $\calH_n$-measurable estimator $\widehat u$ satisfies
  $\bbE_{\pi,\mu}\norm{\widehat u-u^\star}^2\le c_1$, then there is an
  event $\calS_{n,\mu}\in\sigma(\calH_n)$ of probability at least $1/2$
  on which $\overline{\calI}_{n,\mu}^u$ has at least $d/2$ eigenvalues
  of size at least $d$.

  \item There is an $\calH_n$-measurable orthogonal projection $Q_n$ such
  that, simultaneously,
  \begin{align}
    \bbE_{\pi,\mu}\Tr(I_d-Q_n)
      &\le\min\left\{
        \frac1{100}\left(
          d+\frac{\lambda^2\eps^2n}{r^2d^2}\right),
        \frac1{50}\left(
          d+C\frac{\lambda^2\eps^2}{r^2d^2}\mathfrak D_n
        \right)
      \right\},
        \label{eq:auxiliary-codimension}\\
    \bbE_{\pi,\mu}\Tr(\overline{\calI}_{n,\mu}^uQ_n)
      &\le C\lambda^2\eps^2(r^2+e^{-cd})n.
        \label{eq:auxiliary-tension}
  \end{align}
\end{enumerate}
\end{lemma}

\begin{proof}
The proofs are similar to those in \Cref{lem:posterior-target-spread,lem:tension-estimation,lem:tension-regret}. We only highlight the differences.

\paragraph{Parameter derivatives.}
If $\norm{a_t}\ge1$, then $\mu_{\theta^\star}(a_t)=\Gamma(a_t)$ and all
parameter derivatives vanish.  If $\norm{a_t}<1$, put
\begin{equation*}
    s_t=\lambda f_{\theta^\star}(a_t)-\Gamma(a_t)-2\lambda.
\end{equation*}
The derivatives of the second softmax branch and the chain rule give
\begin{equation*}
\begin{aligned}
  \nabla_{u^\star}\mu_{\theta^\star}(a_t)
    &=-\lambda\eps\Psi'(s_t)p_t\aone_t,\\
  \nabla_{u^\star}^2\mu_{\theta^\star}(a_t)
    &=\left[\lambda\eps\Psi'(s_t)p_t(1-p_t)
      +\lambda^2\eps^2\Psi''(s_t)p_t^2\right]
      \aone_t{\aone_t}^\top.
\end{aligned}
\end{equation*}
By \cref{eq:smoothed-hinge-properties} and $\eps\le1/2$, the scalar in
brackets is between $0$ and $\lambda\eps p_t$. As $\norm{\aone_t}\le1$ for an in-ball action, this proves
\begin{equation}
\begin{gathered}
  \nabla_{u^\star}\mu_{\theta^\star}(a_t)
    \nabla_{u^\star}\mu_{\theta^\star}(a_t)^\top, \bigl(\nabla_{u^\star}^2\mu_{\theta^\star}(a_t)\bigr)^2
    \preceq\calI_{t,\mu}^u-\calI_{t-1,\mu}^u.
\end{gathered}
    \label{eq:auxiliary-target-derivative-bounds}
\end{equation}
Similarly,
\begin{equation}
  \norm{\nabla_{W^\star}\mu_{\theta^\star}(a_t)}_F
  \le \frac{\lambda\eps}{r}
    \Ind_{\{\norm{a_t}<1\}}.
    \label{eq:auxiliary-map-derivative-bound}
\end{equation}

\paragraph{First assertion.}
Let $\rho_{n,u}^{\mu}$ be the marginal posterior density of $u^\star$ given
$\calH_n$, integrating out $W^\star$.  The same differentiation of the
marginal posterior as in the proof of \Cref{lem:likelihood-Fisher} gives
\begin{equation}
  J(\rho_{n,u}^{\mu})
  \preceq16dI_d+\overline{\calI}_{n,\mu}^u
    -\overline{\calC}_{n,\mu}^u,
    \label{eq:auxiliary-posterior-fisher-first}
\end{equation}
where
\begin{equation*}
  \calC_{n,\mu}^u
  =\sum_{t=1}^n Z_t\nabla_{u^\star}^2
    \mu_{\theta^\star}(a_t),
  \qquad
  \overline{\calC}_{n,\mu}^u
  =\bbE_{\pi,\mu}[\calC_{n,\mu}^u\mid\calH_n].
\end{equation*}

\noindent Relative to the filtration which reveals
$(W^\star,u^\star,\calH_{t-1},a_t)$ before $Z_t$, the Hessian
$B_t=\nabla_{u^\star}^2\mu_{\theta^\star}(a_t)$ is predictable and
$Z_t$ is standard Gaussian.  Thus the matrix Laplace-transform argument
used in Section~3, together with the second inequality of
\cref{eq:auxiliary-target-derivative-bounds}, gives
\begin{equation*}
  \bbE_{\pi,\mu}\Tr\exp\left(
    -\eta\calC_{n,\mu}^u-\frac{\eta^2}{2}\calI_{n,\mu}^u
  \right)\le d.
\end{equation*}
Therefore, the same argument in \Cref{lem:likelihood-Fisher} gives
\begin{equation}
  J(\rho_{n,u}^{\mu})
  \preceq C(dI_d+\overline{\calI}_{n,\mu}^u)
  \quad\text{with probability at least }3/4.
    \label{eq:auxiliary-posterior-fisher-upper}
\end{equation}

\noindent The rest now follows from the same proof of \Cref{lem:posterior-target-spread}. 

\paragraph{Second assertion.}
Let $\rho_{t,W}^{\mu}$ be the marginal posterior density of $W^\star$ in
the auxiliary model and define
\begin{equation*}
    K_t=\frac1{d^2}\Trout(J(\rho_{t,W}^{\mu})).
\end{equation*}
The convexity of Fisher information again makes $(K_t)$ a positive
semidefinite submartingale.  Fisher's identity and
\cref{eq:auxiliary-map-derivative-bound} give the analogue of
\cref{eq:info-constraint}:
\begin{equation}
  \bbE_{\pi,\mu}\Tr(K_n)
  \le d+\frac{\lambda^2\eps^2n}{r^2d^2}.
    \label{eq:auxiliary-info-constraint}
\end{equation}
The Fisher small-ball argument used to prove
\cref{eq:anti-concentration} depends only on the current marginal posterior
of $W^\star$, so it applies without change.  For in-ball actions it yields
\begin{equation}
\begin{aligned}
  {\aone_t}^\top K_{t-1}^{-1}\aone_t\,
  \bbE_{\pi,\mu}\left[
    p_t^2\Ind_{\{\norm{u^\star}\le1/2\}}
    \mid\calH_{t-1},a_t\right]
  \le Cr^2.
\end{aligned}
    \label{eq:auxiliary-anti-concentration}
\end{equation}
\noindent These establish the counterparts of \Cref{lem:Kt-property}. The second assertion now follows from the same proofs of \Cref{lem:tension-estimation,lem:tension-regret}. 
\end{proof}

\subsection{Proof of \Cref{thm:unconstrained-convex-bandit-regret-lower-bound}}

\paragraph{Estimation complexity.}
Fix $\eps\in(0,1/2]$ and set
\begin{equation*}
    r=\max\{d^{-1/4},16\eps\}.
\end{equation*}
Then \cref{eq:unconstrained-parameter-condition} holds.  Consider an
arbitrary learner in the true model
$\theta^\star\sim\pi_{\mathrm b}$,
$r_t=g_{\theta^\star}(a_t)+Z_t$, and let
$\widehat u=\Pi^1(\widehat a)$.  Conditional on
$\calE_{\mathrm b}$, the entire trajectory in the Gaussian auxiliary model
has the same law as the trajectory in this bounded-prior model, since
$\mu_\theta=g_\theta$ at every action on $\calE_{\mathrm b}$.  Moreover,
$\norm{\widehat u}\le1$.  Splitting the Gaussian risk over
$\calE_{\mathrm b}$ and its complement, as in
\Cref{lem:conditioning-transfer}, therefore gives
\begin{equation}
  \bbE_{\pi,\mu}\norm{\widehat u-u^\star}^2
  \le \bbE_{\pi_{\mathrm b},g}\norm{\widehat u-u^\star}^2
    +Ce^{-cd}.
    \label{eq:auxiliary-estimation-transfer}
\end{equation}

\noindent Suppose, toward a contradiction, that the Bayes optimization error under
$\pi_{\mathrm b}$ is smaller than $c_0\eps$, for a sufficiently small
universal $c_0$.  By \cref{eq:conic-two-loss-domination},
\begin{equation*}
  \bbE_{\pi_{\mathrm b},g}\norm{\widehat u-u^\star}^2
  \le Cc_0/\lambda.
\end{equation*}
Thus, for small $c_0$ and large $d$, \cref{eq:auxiliary-estimation-transfer}
and the first part of \Cref{lem:auxiliary-mean-information} produce an event
$\calS_{n,\mu}$ of probability at least $1/2$ on which
$\overline{\calI}_{n,\mu}^u$ has at least $d/2$ eigenvalues at least $d$.

\medskip
\noindent Let $Q_n$ be the projection in the second part of that lemma.  If
\begin{equation}
    n<c_2\min\left\{
      \frac{d^{5/2}}{\eps^2},\frac{d^2}{\eps^4}
    \right\},
    \label{eq:unconstrained-revised-estimation-budget}
\end{equation}
then a direct check in the two cases
$r=d^{-1/4}$ and $r=16\eps$ shows that
\begin{equation*}
  \frac{\lambda^2\eps^2n}{r^2d^2}\le c d,
  \qquad
  \lambda^2\eps^2r^2n\le c d^2,
\end{equation*}
where $c$ can be made arbitrarily small by decreasing $c_2$.  Consequently,
the first term in the minimum in \cref{eq:auxiliary-codimension} gives
$\bbE\Tr(I_d-Q_n)\le d/16$.  Hence
$\operatorname{rank}(Q_n)\ge3d/4$ with probability at least $3/4$.
Intersecting this event with $\calS_{n,\mu}$ and using von Neumann's trace inequality gives
\begin{equation}
  \bbE_{\pi,\mu}\Tr(\overline{\calI}_{n,\mu}^uQ_n)
  \ge\frac{d^2}{16}.
    \label{eq:auxiliary-target-lower-estimation}
\end{equation}
On the other hand, \cref{eq:auxiliary-tension} gives
\begin{equation*}
  \bbE_{\pi,\mu}\Tr(\overline{\calI}_{n,\mu}^uQ_n)
  \le C\lambda^2\eps^2(r^2+e^{-cd})n
  <\frac{d^2}{16}
\end{equation*}
after decreasing $c_2$.  This contradicts
\cref{eq:auxiliary-target-lower-estimation}.  Consequently the Bayes error
is at least $c_0\eps$, so at least one
$g_\theta\in\calF_{2d}(\bbR^{2d})$ has expected optimization error at least
$c_0\eps$.

\paragraph{Regret lower bound.}
Fix a horizon $n\ge Cd^2$.  For a small universal constant $\kappa>0$,
take
\begin{equation}
 (\eps_n,r)=
 \begin{cases}
   \bigl(\kappa\sqrt d\,n^{-1/4},
     16\kappa\sqrt d\,n^{-1/4}\bigr),
       &d^2\le n\le d^{10/3},\\
   \bigl(\kappa d^{4/3}n^{-1/2},
     16\kappa d^{-1/3}\bigr),
       &n>d^{10/3}.
 \end{cases}
    \label{eq:unconstrained-regret-parameters}
\end{equation}
In both regimes $\eps_n\le1/2$ and $\eps_n/r\le1/16$, after decreasing
$\kappa$ if necessary, so the preceding construction applies.

\medskip
\noindent Let $\mathfrak R_n^{\mathrm b}$ be the Bayes regret of an arbitrary policy
against $g_{\theta^\star}$ under $\pi_{\mathrm b}$.  We argue by
contradiction, assuming
\begin{equation}
    \mathfrak R_n^{\mathrm b}<c\lambda\eps_n n
    \label{eq:assumed-small-unconstrained-regret}
\end{equation}
for a sufficiently small universal $c>0$.  Define
\begin{equation*}
    \widehat u=\frac1n\sum_{t=1}^n\Pi^1(a_t).
\end{equation*}
The second term in \cref{eq:conic-two-loss-domination} and Jensen's inequality
give
\begin{equation*}
  \bbE_{\pi_{\mathrm b},g}\norm{\widehat u-u^\star}^2
  \le\frac{C\mathfrak R_n^{\mathrm b}}{\lambda\eps_n n}.
\end{equation*}
The estimate $\widehat u$ lies in the unit ball, so the conditioning argument
in \cref{eq:auxiliary-estimation-transfer} applies to it as well.  Under
\cref{eq:assumed-small-unconstrained-regret}, the first part of
\Cref{lem:auxiliary-mean-information} therefore yields an event
$\calS_{n,\mu}$ of probability at least $1/2$ on which
$\overline{\calI}_{n,\mu}^u$ has at least $d/2$ eigenvalues at least $d$.

\medskip
\noindent We next control the codimension of the regret projection.  In the
bounded-prior model, the first term in
\cref{eq:conic-two-loss-domination} gives, for every in-ball action,
\begin{equation*}
  g_{\theta^\star}(a_t)-g_{\theta^\star}(a_{\theta^\star})
  \ge c\lambda\frac{\eps_n}{r}D_t.
\end{equation*}
Because the summands in \cref{eq:auxiliary-clipped-tube-cost} are bounded by
one, conditioning on $\calE_{\mathrm b}$ gives the uniform transfer bound
\begin{equation}
  \mathfrak D_n
  \le C\frac{r}{\lambda\eps_n}\mathfrak R_n^{\mathrm b}
    +Cne^{-cd}.
    \label{eq:auxiliary-tube-transfer}
\end{equation}
This is the point of clipping $D_t$: the exceptional-event remainder is
$ne^{-cd}$, with no factor growing like $r/\eps_n$.

\medskip
\noindent Let $Q_n$ be the projection in the second part of
\Cref{lem:auxiliary-mean-information}.  Substituting
\cref{eq:auxiliary-tube-transfer} into
the second term in the minimum in \cref{eq:auxiliary-codimension} gives
\begin{equation}
\begin{aligned}
  \bbE_{\pi,\mu}\Tr(I_d-Q_n)
  \le\frac1{50}\left[d
    +C\frac{\lambda\eps_n}{rd^2}\mathfrak R_n^{\mathrm b}
    +C\frac{\lambda^2\eps_n^2n}{r^2d^2}e^{-cd}
  \right].
\end{aligned}
    \label{eq:auxiliary-regret-codim-after-transfer}
\end{equation}
Using \cref{eq:assumed-small-unconstrained-regret} and
\cref{eq:unconstrained-regret-parameters}, the second term in brackets is
at most $C\kappa d$: in the first regime this follows from
$n^{3/4}\le d^{5/2}$, and in the second regime the term is exactly of order
$\kappa d$.  The third term is at most
$Cd^{4/3}e^{-cd}$.  Thus, after decreasing $c$ and $\kappa$,
\cref{eq:auxiliary-regret-codim-after-transfer} is at most $d/16$ for
large $d$.  As before,
$\operatorname{rank}(Q_n)\ge3d/4$ with probability at least $3/4$, and
intersection with $\calS_{n,\mu}$ gives
\begin{equation}
  \bbE_{\pi,\mu}\Tr(\overline{\calI}_{n,\mu}^uQ_n)
  \ge\frac{d^2}{16}.
    \label{eq:auxiliary-target-lower-regret}
\end{equation}
The tension inequality \cref{eq:auxiliary-tension}, however, yields
\begin{align*}
  \bbE_{\pi,\mu}\Tr(\overline{\calI}_{n,\mu}^uQ_n)
  &\le C\lambda^2\eps_n^2(r^2+e^{-cd})n\\
  &\le C\lambda^2\left(\kappa^4d^2
      +\kappa^2d^{8/3}e^{-cd}\right)
  <\frac{d^2}{16}
\end{align*}
when $\kappa$ is sufficiently small.  This contradicts
\cref{eq:auxiliary-target-lower-regret}.  We have proved
\begin{equation*}
  \mathfrak R_n^{\mathrm b}
  \ge c\lambda\eps_n n
  \ge c'\min\left\{
    \sqrt d\,n^{3/4},d^{4/3}\sqrt n
  \right\}.
\end{equation*}
Since Bayes regret lower bounds minimax regret, this proves the asserted
regret bound.